\documentclass[11pt]{article}

\usepackage{amsmath,amsthm}
\usepackage{microtype}
\usepackage{palatino}
\usepackage{graphicx}
\usepackage{algorithm,algpseudocode}
\usepackage[caption=false,font=footnotesize]{subfig}

\def\x{\boldsymbol{x}}
\def\y{\boldsymbol{y}}
\def\u{\boldsymbol{u}}

\newtheorem{theorem}{Theorem}[section]
\newtheorem{proposition}[theorem]{Proposition}

\begin{document}

\title{Acceleration of the shiftable $O(1)$ algorithm for bilateral filtering and non-local means}
\author{Kunal~N.~Chaudhury\thanks{Correspondence: kchaudhu@math.princeton.edu. 
}}

\maketitle

\begin{abstract}
A direct implementation of the bilateral filter \cite{Tomasi1998} requires $O(\sigma_s^2)$ operations per pixel, where $\sigma_s$ is the (effective) width of the spatial kernel. A fast implementation of the bilateral filter 
was recently proposed in \cite{KcDsMu2011} that required $O(1)$ operations per pixel with respect to $\sigma_s$. This was done by using trigonometric functions for the range kernel of the bilateral filter, and by exploiting 
their so-called \textit{shiftability} property. In particular, a fast implementation of the Gaussian bilateral filter was realized by approximating the Gaussian range kernel using raised cosines. 
Later, it was demonstrated in \cite{Kc2011} that this idea could be extended to a larger class of filters, including the popular non-local means filter \cite{Baudes2005,Yaroslavsky1985}. 
As already observed in \cite{KcDsMu2011}, a flip side of this approach was that the run time depended on the width $\sigma_r$ of the range kernel. For an image with dynamic range $[0,T]$, the run 
time scaled as $O(T^2/\sigma^2_r)$ with $\sigma_r$. This made it difficult to implement narrow range kernels, particularly for images with large dynamic range. In this paper, we discuss this problem, and propose some simple steps 
to accelerate the implementation, in general, and for small $\sigma_r$ in particular. We provide some experimental results to demonstrate the acceleration that is achieved using these modifications.
\end{abstract}


\section{Introduction}

The bilateral filter is an edge-preserving diffusion filter, which was introduced by Tomasi et al. in \cite{Tomasi1998}. The edge-preserving property comes from the use of a range kernel (along with the spatial kernel) that
is used to control the diffusion in the vicinity of edges. In this work, we will focus on the Gaussian bilateral filter where both the spatial and range kernels are Gaussian \cite{Tomasi1998}. This is given by
\begin{equation}
\label{BF}
\tilde{f}(\x)=\frac{1}{\eta} \int_{\Omega} g_{\sigma_s}(\x - \y) \ g_{\sigma_r}(f(\x - \y) - f(\x)) \ f(\x - \y) \ d\y
\end{equation}
where
\begin{equation*}
\eta = \int_{\Omega}  g_{\sigma_s}(\x - \y) \ g_{\sigma_r}(f(\x - \y) -f(\x))  \ d\y.
\end{equation*}
Here, $g_{\sigma_s}(\x)$ is the centered Gaussian distribution on the plane with variance $\sigma_s^2$, and $g_{\sigma_r}(s)$ is the one-dimensional Gaussian distribution with variance $\sigma_r^2$;
$\Omega$ is the support of $g_{\sigma_s}(\x)$ over which the averaging takes place. We call $g_{\sigma_s}(\x)$ and $g_{\sigma_r}(s)$ the spatial and the range kernel.

The range kernel is controlled by the local distribution of intensity. Sharp discontinuities (jumps) in intensity typically occur in the vicinity of edges. This is
picked up by the range kernel, which is then used to inhibit the spatial diffusion. On the other hand, the range kernel becomes inoperative in regions with smooth variations in intensity. 
The spatial kernel then takes over, and the bilateral filter behaves as a standard diffusion filter. Together, the spatial and range kernels perform smoothing in homogeneous regions, while preserving edges at the same time
\cite{Tomasi1998}. 

The bilateral filter has found widespread use in several image processing, computer graphics, and computer vision applications \cite{Bennet, VideoAbstraction, Ramanath, Xiao, Yang_BLF_stereo}; see \cite{DurandApps} for further applications. 
More recently, the bilateral filter was extended by Baudes et al. \cite{Baudes2005} in the form of the non-local means filter, where the similarity between pixels is measured using patches centered around the pixel.

\subsection{Fast bilater filter}

The direct implementation of \eqref{BF} is computationally intensive, especially when $\sigma_s$ is large ($\sigma_r$ has no effect on the run time in this case). In particular, the direct implementation 
requires $O(\sigma_s^2)$ operations per pixel. This makes the filter slow for real-time applications. Several efficient algorithms have been proposed in the past for implementing the filter in real time, 
e.g., see \cite{Durand,Weiss,Paris,Yang,Gunturk2011,Adams2010}. In \cite{Porikli}, Porikli demonstrated for the first time that the bilateral filter could be implemented using $O(1)$ operations per pixel (with respect to $\sigma_s$).
This was done for two different settings: $(a)$ Spatial box filter and arbitrary range filter, and $(b)$ Arbitrary spatial filter and polynomial range filter. 
The author extended $(b)$ to the Gaussian bilateral filter in \eqref{BF} by approximating $g_{\sigma_r}(s)$ with its Taylor polynomial. The run time
of this approximation was linear in the order of the polynomial. The problem with Taylor polynomials, however, is that they provide good approximations of $g_{\sigma_r}(s)$ only locally around the origin. In particular, they have the following drawbacks:

\begin{itemize}
\item  Taylor polynomials are not guaranteed to be positive and monotonic away from the origin, where the approximation is poor. Moreover, they tend to blow up at the tails. 
 
\item  It is difficult to approximate $g_{\sigma_r}(s)$ using the Taylor expansion when $\sigma_r$ is small. In particular, a large order polynomial is required to get a good approximation of a narrow Gaussian, and this 
considerably increases the run time of the algorithm.  
\end{itemize}

The first of these problems was addressed in \cite{KcDsMu2011}. In this paper, the authors observed that it is important that the kernel used to approximate $g_{\sigma_r}(s)$ be positive, monotonic, and symmetric. While it is easy to ensure symmetry,
the other two properties are hard to enforce using Taylor approximations. It was noticed that, in the absence of these properties, the bilateral filter in \cite{Porikli} created strange artifacts in the processed 
image (cf. Figure 3 in \cite{KcDsMu2011}). The authors proposed to fix this problem using the family of raised cosines, namely, functions of the form
\begin{equation}
\label{RC}
\phi(s)=\Big [ \cos \left(\frac{\pi s}{2T} \right) \Big]^N \qquad (-T \leq s \leq T).
\end{equation}
Here $N$ is the order of the kernel, which controls the width of $\phi(s)$. The kernel can be made narrow by increasing $N$. 

The key parameter in \eqref{RC} is the quantity $T$. The idea here is that $[\cos(s)]^N$ is guaranteed to be positive and monotonic provided that $s$ is restricted to the interval $[-\pi/2, \pi/2]$. 
Note that the argument $s$ in \eqref{RC} takes on the values $|f(\x -\y) - f(\x)|$ as $\x$ and $\y$ varies over the image. Therefore, by letting 
\begin{equation*}
T = \max_{\x} \ \max_{\y \in \Omega} \ |f(\x-\y) - f(\x)|, 
\end{equation*}
one could guarantee $\phi(s)$ to be positive and monotonic over $[-T,T]$. In \cite{KcDsMu2011}, $T$ was simply set to the maximum dynamic range, for example, $255$ for grayscale images. 
We refer the readers to Figure 2 in \cite{KcDsMu2011} for a comparison of \eqref{RC} and the polynomial kernels used in \cite{Porikli}. It was later observed in \cite{Kc2011} that polynomials could also be used for the same purpose. 
The polynomials suggested were of the form
\begin{equation}
\label{RP}
\phi(s)=\left( 1 - \frac{s^2}{T^2}  \right)^N \qquad (-T \leq s \leq T).
\end{equation}

\subsection{Fast $O(1)$ implementation using shiftable kernels}

For completeness, we now explain how the above kernels can be used to compute \eqref{BF} using $O(1)$ operations. As observed in \cite{Kc2011}, \eqref{RC} and \eqref{RP} are essentially the simplest kernels that have the so-called property 
of \textit{shiftability}. This means that, for a given $N$, we can find a fixed set of basis functions $\phi_1(s),\ldots,\phi_{N}(s)$ and coefficients $c_1,\ldots,c_{N}$, so that for any translation $\tau$, we can write
\begin{equation}
\phi(s - \tau) = c_1(\boldsymbol{\tau}) \phi_1(s)+ \cdots + c_N(\boldsymbol{\tau}) \phi_{N}(s).
\end{equation}
The coefficients depend continuously on $\tau$, but the basis functions have no dependence on $\tau$. For \eqref{RC}, both the basis functions and coefficients are cosines, while they are polynomials for \eqref{RP}.
This shiftability property is at the heart of the $O(1)$ algorithm. Let $\overline{f}(\x)$ denote the output of the Gaussian filter $g_{\sigma_s}(\x)$ with neighborhood $\Omega$, 
\begin{equation}
\label{GF}
\overline{f}(\x) = \int_{\Omega} g_{\sigma_s}(\x-\y) f(\y) \ d\y.
\end{equation}
Note that, by replacing $g_{\sigma_r}(s)$ with $\phi(s)$, we can write \eqref{BF} as
\begin{equation}
\label{f1}
\tilde{f}(\x) =  \frac{1}{\eta}   \Big[ c_1(f(\x)) \ \overline{F_1}(\x) + \cdots  c_{N}(f(\x)) \ \overline{F_{N}}(\x) \Big],
\end{equation}
where we have set $F_i(\x) = f(\x)\phi_i(f(\x))$. Similarly, by setting $G_i(\x) = \phi_i(f(\x))$, we can write
\begin{equation}
\label{f2}
\eta =  c_1(f(\x)) \ \overline{G_1}(\x) + \cdots + c_{N}(f(\x)) \ \overline{G_{N}}(\x) .
\end{equation}
Now, it is well-known that certain approximation of \eqref{GF} can be computed using just $O(1)$ operations per pixel. These recursive algorithms are based on specialized kernels, such as the box and 
the hat function \cite{Heckbert,Crow}, and the more general class of box splines \cite{kunal_tip}. Putting all these together, we arrive at the following $O(1)$ algorithm for approximating \eqref{BF}:

\begin{enumerate}

\item Fix $N$, and approximate $g_{\sigma_r}(s)$ using \eqref{RC} or \eqref{RP}.

\item For $i =1,2,\ldots,N$, set up the images $F_i(\x) = f(\x)\phi_i(f(\x))$ and $G_i(\x) = \phi_i(f(\x))$, and the coefficients  $c_i(f(\x))$ .

\item Use a recursive $O(1)$ algorithm to compute each $\overline{F_i}(\x)$ and $\overline{G_i}(\x)$.

\item Plug these into \eqref{f1} and \eqref{f2} to get the filtered image.

\end{enumerate}

It is clear that better approximations are obtained when $N$ is large. On the other hand, the run time scales linearly with $N$. One key advantage of the above algorithm, however, is 
that the $\overline{F_i}(\x)$ and $\overline{G_i}(\x)$ can be computed in parallel. For small orders ($N < 10$), the serial implementation is found to be comparable, 
and often better, than the state-of-the-art algorithms. The parallel implementation, however, turns out to be much faster than the competing algorithms, at least for $N < 50$.  Henceforth, we will refer to
the above algorithm as \texttt{SHIFTABLE-BF}, the shiftable bilateral filter.

\subsection{Gaussian approximation for small $\sigma_r$}

This brings us to the question as to whether we can always work with, say, $N < 50$ basis functions, in \texttt{SHIFTABLE-BF}? To answer this question, we must explain in some detail step (1) of the algorithm, where we 
approximate the Gaussian range kernel,
\begin{equation*}
g_{\sigma_r}(s) = \exp\left(-\frac{s^2}{2\sigma_r^2}\right), 
\end{equation*}
on the interval $[-T,T]$. This could be done either using \eqref{RC},
\begin{equation}
\label{RCconv}
g_{\sigma_r}(s) =  \lim_{N \longrightarrow \infty} \left[\cos \left( \frac{s}{ \sqrt N \sigma}\right)\right]^N,
\end{equation}
or, using \eqref{RP},
\begin{equation}
\label{RPconv}
g_{\sigma_r}(s) =  \lim_{N \longrightarrow \infty} \left( 1 - \frac{s^2}{2 N \sigma^2}  \right)^N.
\end{equation}
These approximations were proposed in \cite{KcDsMu2011,Kc2011}. Note that, we have to rescale \eqref{RC} by $\sqrt{N}$, and \eqref{RP} by $N$, to get to the right limit. On the other hand, to ensure positivity and monotonicity, we 
need to guarantee that the arguments of \eqref{RCconv} and \eqref{RPconv} are in the intervals $[-\pi/2, \pi/2]$ and $[0,1]$. A simple calculation shows that this is the case provided that $N$ is larger than 
$N_0 = 4T^2/\pi^2 \sigma_r^2 = 0. 405 (T/\sigma_r)^2$ for the former, and $N_0 = 0. 5 (T/\sigma_r)^2$ for the latter. In other words, it is not sufficient to set $N$ large -- it must be at least be as large as $N_0$. 
In Table \ref{table1}, we give the values of $N_0$ for different values of $\sigma_r$ when $T=255$. It is seen that $N_0$ gets impracticable large for $\sigma_r < 30$. 
This does not come as a surprise since it is well-known that one requires a large number of trigonometric functions (or polynomials) to closely approximate a narrow Gaussian on a large interval.  
As pointed out earlier, this was also one of the problems in \cite{Porikli}.

\begin{table}[!htbp]
\caption{The threshold $N_0$ for different $\sigma_r$ ($T =255$).} 
\label{spread} 
\centering
\begin{tabular}{|c|cccccccc|}
\hline
$\sigma_r$   & 5 &  10 & 20 &  30  & 40 &  60  &  80 & 100  \\
\hline
$N_0$        & 1053  &  263  &  66 &  29 &  16  &  7 & 4 & 3    \\
\hline
\end{tabular}
\label{table1}
\end{table}

\subsection{Present Contributions}

In this paper, we address the above problem, namely that $N_0$ grows as $O(T^2/\sigma_r^2)$ with $\sigma_r$. In Section \ref{algo for T}, we propose a fast algorithm for determining $T$ exactly. Besides cutting down $N_0$, this 
is essential for determining the (local) dynamic range of a grayscale image that has been deformed, e.g., by additive noise. Setting $T=255$ in this case can lead to artifacts in the processed image. 
Next, in Section \ref{algo for sigma}, we provide a simple and practical means of reducing the order, which leads to quite dramatic reductions in the run time of \texttt{SHIFTABLE-BF}. 
These modifications are also applicable to the shiftable algorithms proposed in \cite{KcDsMu2011,Kc2011}. Finally, in Section \ref{experiments}, we provide
some experimental results to demonstrate the acceleration that is achieved using these modifications. We also compare our algorithm with the Porikli's algorithms \cite{Porikli}, both in terms of speed and accuracy.

\section{Fast algorithm for finding $T$}
\label{algo for T}

For the rest of the discussion, we work with finite-sized images (bounded $\Omega$) on the Cartesian grid. We continue to use $\x$ and $\y$ to denote points on the grid. The integral in \eqref{BF} is simply replaced by a 
finite sum over $\Omega$. We will use the norm $\lVert \x \rVert = |x_1| + |x_2|$, where $\x = (x_1,x_2)$. Without loss of generality, we assume that $\Omega$ is a square neighborhood, 
that is, $\Omega = \{\x : \lVert \x \rVert \leq R\}$ where, say, $R=3\sigma_s$ (if $\Omega$ is not rectangular, we take the smallest rectangle containing $\Omega$). 

Note that, for a given $\sigma_r$, we can cut down $N_0$ by using a tight estimate for 
\begin{equation}
\label{bound}
T = \max_{\x} \ \max_{\lVert \y \rVert \leq R} \ |f(\x-\y) - f(\x)|. 
\end{equation}
The smaller the estimate, the lower is the threshold $N_0$. The point is that the worst-case estimate $T= 255$ is often rather loose for grayscale images. For example, we give
the exact values of $T$ for a test image in Table \ref{table2}, computed at different values of $\sigma_s$. We also give the time required to compute $T$.

\begin{table}[!htbp]
\caption{Exact values of $T$ for the standard $512 \times 512$ \textit{Lena}. Also shown is the time needed to compute $T$ using Matlab.} 
\label{spread} 
\centering
\begin{tabular}{|c|ccccccc|}
\hline
$\sigma_s$                         &  1   &  3     &  5    &  10    & 15    &  20 & 30  \\
\hline
$T$                                & 153  &  205   &  208  &  210   & 211   & 215 & 215\\
\hline
$\mathrm{time \ (sec)}$            &  2.06    &  2.13    &  2.33   &  2.71    & 3.26    & 4.01 & 5.60 \\
\hline
\end{tabular}
\label{table2}
\end{table}

It is seen that the exact values of $T$ are indeed much less than the worst-case estimate, particularly for small $\sigma_s$. For $\sigma_s = 3$, $T$ is only about $205$. Even for $\sigma_s$ as large  as $30$, $T$ is about 
$215$. Consider the bilateral filter with $\sigma_s = 10$ and $\sigma_r = 10$. From Table \ref{table1}, $N_0 = 263$ using $T = 255$. However, using the exact value $T = 210$, we can bring this down
to $263 \cdot (210/255)^2 \approx 178$, a reduction by almost $100$. For smaller values of $\sigma_r$, this gain is even more drastic. However, notice the time required to compute $T$ in Table \ref{table1}. This increases 
quickly with the increase in $\sigma_s$ (in fact, scales as $O(R^2)$). Experiments show us that, for large $\sigma_s$, this is comparable to the time required to compute the bilateral filter. It would thus help to have an $O(1)$ algorithm for computing $T$.
Motivated by our previous work on filtering using running sums \cite{kunal_tip}, we recently devised an algorithm that does exactly this. We later found that the algorithm had already been 
discovered two decades back in a different context \cite{Herk1992,GilWerman1992}.

Our algorithm is based on the following observations. First, note that we can take out the modulus from \eqref{bound} using symmetry. 

\begin{proposition}[Simplification]
\begin{equation}
\label{MaxFilter}
T = \max_{\x}  \left[ f(\x) -  \max_{\lVert \y \rVert \leq R} \ f(\x-\y) \right]. 
\end{equation}
\end{proposition}

\begin{proof}
This follows from the observations that $|t| = \max(t,-t)$, and that $\lVert \x - \y \rVert \leq R$ is symmetric in $\x$ and $\y$. Moreover, note that the operation that takes two numbers $a$ and $b$ and returns $\max(a,b)$
is associative. Using associativity, we can write \eqref{bound} as $T = \max(T_{+}, T_{-})$, where
\begin{equation}
T_{+} = \max_{\x} \left[ f(\x) -  \ \max_{\lVert \y \rVert \leq R} \ f(\x-\y) \right], 
\end{equation}
and
\begin{equation}
T_{-} =  \max_{\x} \left[ \max_{\lVert \y \rVert \leq R} \ f(\x-\y) - f(\x) \right], 
\end{equation}
We claim that $T_{+} = T_{-}$, so that we need not compute them separately. Indeed, suppose that the first maximum is attained at $\x_0$ and $\y_0$, that is, $T_{+} = f(\x_0) - f(\x_0-\y_0)$. Taking $\x = \x_0 - \y_0$
and $\y = - \y_0$, and noting that $\lVert \x-\y \rVert \leq R$, we must have
\begin{equation*}
T_{-} \geq  f(\x-\y) - f(\x) = f(\x_0) - f(\x_0-\y_0) = T_{+}. 
\end{equation*}
By an identical argument, $T_{+} \geq T_{-}$, and the proposition follows.
\end{proof}

The problem is now reduced to that of computing the windowed maximums in \eqref{MaxFilter}. A direct computation would still require $O(R^2)$ comparisons. It turns out that we can do this very fast (no matter how large is $R$)
 by exploiting the overlap between adjacent windows. This is done by adapting the so-called \texttt{MAX-FILTER} algorithm.

\begin{proposition}[Max-Filter Algorithm]
There is an $O(1)$ algorithm for computing
\begin{equation*}
\max_{\lVert \y \rVert \leq R}   \ f(\x-\y) 
\end{equation*}
at every $\x$.
\end{proposition}

This algorithm was first proposed by van Herk, and Gil and Werman \cite{Herk1992,GilWerman1992}. It is clear that since the search domain $\Omega$ is separable, it suffices to solve the problem in one dimension. 
The problem in two-dimensions can be solved simply by iterating the one-dimensional \texttt{MAX-FILTER} along each dimension. From \eqref{MaxFilter}, we arrive at Algorithm \ref{algo1} for computing $T$ with $O(1)$ operations. We 
note that Algorithm \ref{algo1} does not actually compute $\max \ \{ |f(\x-\y) -f(\x)|: \lVert \y \rVert \leq R \}$ at every $\x$. It only has access to the distribution of the maximums. For completeness, we have 
explained the \texttt{MAX-FILTER} algorithm in the Appendix. For further details, we refer the readers to \cite{Herk1992,GilWerman1992}.

\begin{algorithm}[!htb]
\caption{Fast algorithm for computing $T$}
\label{algo1}
\begin{algorithmic}
     \State \textbf{Input}: Image $f(\x)$ and  $R$.
      \State \textbf{Return}: $T$ as in \eqref{bound}.
     \State 1. Set $R$ as window of \texttt{MAX-FILTER}.
     \State 2. Apply \texttt{MAX-FILTER} along each row of $f(\x)$; return image $m(\x)$. 
     \State 3. Apply \texttt{MAX-FILTER} along each column of $m(\x)$; return image $M(\x)$.
     \State 4. Set $T$ as the maximum of $f(\x) - M(\x)$ over all $\x$.
\end{algorithmic}
\end{algorithm}	

\section{Acceleration using truncations}
\label{algo for sigma}

\begin{figure}
\centering
\includegraphics[width=1.0\linewidth]{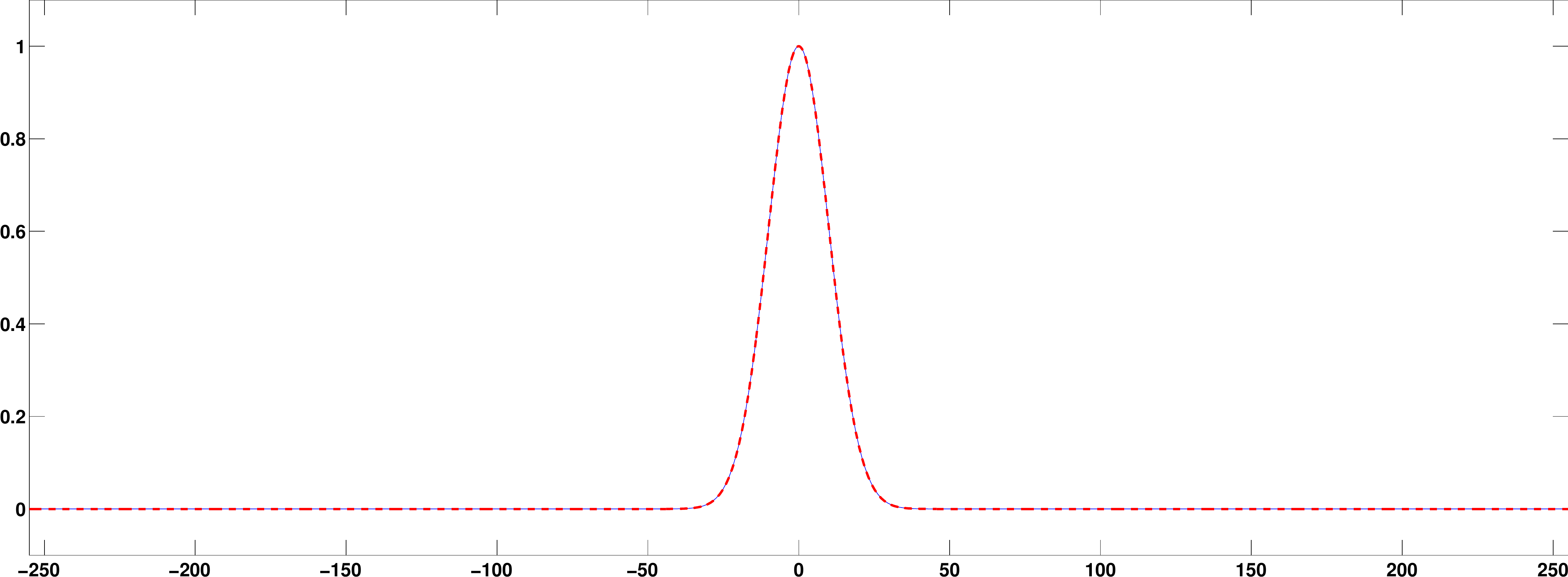} 
\includegraphics[width=1.0\linewidth]{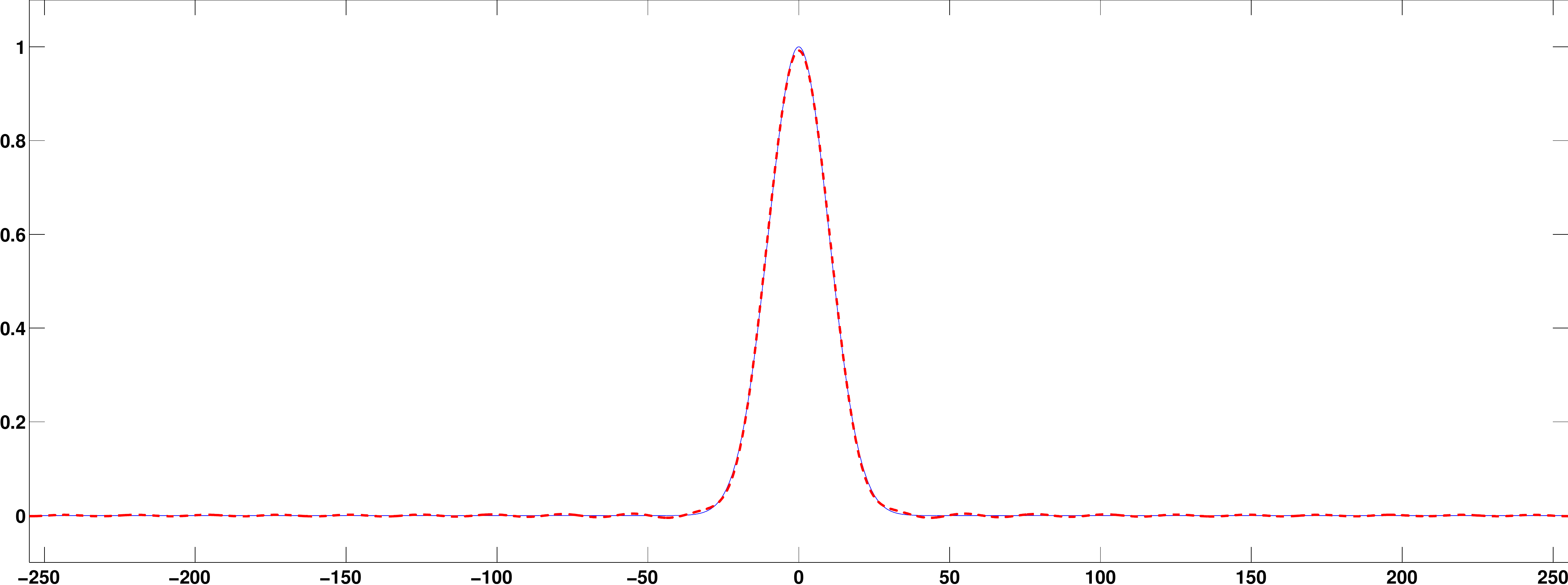} 
\caption{Approximation of the target Gaussian $g_{\sigma_r}(s)$  on the interval $[-255,255]$ using the raised cosine $\phi(s)$ in \eqref{RCconv}. 
We have used $\sigma_r=10$ in this case, for which $N_0=263$. \textbf{Top}: The solid blue line shows the target Gaussian $g_{\sigma_r}(s)$, and the broken red line is the raised cosine $\phi(s)$. 
We truncate $\phi(s)$ to obtain $\phi_{\varepsilon}(s)$ in \eqref{approx}, where we used $\varepsilon=0.005$. In this case, $M = 111$. Thus, a total of $2\times M = 222$ terms are dropped from the series. 
The truncation $\phi_{\varepsilon}(s)$ has only $42$ terms. \textbf{Bottom}: The solid blue line is $g_{\sigma_r}(s)$, and the broken red line is $\phi_{\varepsilon}(s)$. Note that the quality of approximation 
is reasonably good even after discarding almost $84 \%$ terms. As a result of the truncation, small oscillations of size $\varepsilon$ emerge on the tails. The approximation around the origin is positive and monotonic.}
\label{fig1}
\end{figure}

We have seen that, by using the exact value of $T$, we can bring down the run time by $10- 20\%$. Unfortunately, Table \ref{table1} tells us that this alone is not sufficient in the regime $\sigma_r < 15$.
For example, $N_0$ is of the order $10^3$ in the regime $\sigma_r < 5$. So why do we require so many terms in \eqref{RCconv} and \eqref{RPconv} to approximate a narrow Gaussian? This is exactly because we are forcing 
$\phi(s)$ to positive and monotonic on its broad tails, where $g_{\sigma_r}(s)$ is close to zero. For example, consider the approximation in \eqref{RCconv}:
\begin{equation}
\label{expansion}
\phi(s)= \sum_{n=0}^N \ 2^{-N} \binom{N}{n} \cos \left( \frac{ (2n-N) s}{\sqrt N \sigma_r }\right).
\end{equation}
By requiring $N > N_0$, we can guarantee that (1) $\phi(s)$ is close to $g_{\sigma_r}(s)$, and (2) $\phi(s)$ is positive and monotonic over $[-T,T]$. Note, however, that $g_{\sigma_r}(s)$ falls off very fast, and almost 
vanishes outside $\pm 3 \sigma_r$. It turns out that only a few significant terms in \eqref{expansion} contribute to the approximation in the $\pm 3\sigma_r$ region. The rest of the terms have a negligible 
contribution, and are required only to force positivity at the tails. 

\begin{table}[!htbp]
\caption{The threshold $N_0$ before and after truncation, tolerance $\varepsilon = 0.05$ (worst-case setting $T =255$).} 
\label{spread} 
\centering
\begin{tabular}{|c|cccccc|}
\hline
$\sigma_r$                 &  3 &  5  &  8 & 10 & 12 & 15  \\
\hline
$N_0$  (before)          &  2929 &  1053 &  413 &  263 & 184  &  119    \\
\hline
$N_0$  (after)          &  95 &  77 &  53 &  43 &  36 &  29    \\
\hline
$\%$ of terms dropped   & 96   &  92 &  88 &  85  &  82 &  77   \\
\hline
\end{tabular}
\label{table3}
\end{table}

Thanks to expression \eqref{expansion}, it is now straightforward to determine which are the significant terms. Note that the coefficients $2^{-N} \binom{N}{n}$ in \eqref{expansion} are positive and sum up to one. 
 In fact, they are unimodal and closely follow the shape of the target Gaussian. The smallest coefficients are at the tails,  and the largest coefficients are at the center. In particular, the smallest 
 coefficient is $1/2^N$, while the largest one is $2^{-N}[(N/2)!]^{-2}N!$ (assuming $N$ to be even). For large $N$, the latter is approximately 
 $\sqrt{2/\pi N}$ using Stirling's formula. Thus, as $N$ gets large, the coefficients get smaller. What is perhaps significant is that the ratio of the smallest to the largest coefficient is $(\pi N 2^{2N-1})^{-1/2}$, and this 
 keeps shrinking at an exponential rate with $N$. On the other hand, the cosine functions (which act as the interpolating function) are always bounded between $[-1,1]$.

The above observation suggests dropping the small terms on the tail. In particular, for a given tolerance $\varepsilon >0$, let $M=M(\varepsilon)$ be the smallest term for which 
\begin{equation}
\label{trunc}
 \sum_{n=0}^{M+1}  2^{-N} \binom{N}{n}> \varepsilon/2,
\end{equation}
and set
\begin{equation}
\label{approx}
\phi_{\varepsilon}(s) = \sum_{n=M}^{N-M} \ 2^{-N} \binom{N}{n} \cos \left( \frac{ (2n-N) s}{\sqrt N \sigma_r }\right).
\end{equation}
It follows that the error $|\phi(s)-\phi_{\varepsilon}(s)|$ is within $\varepsilon$ for all $-T \leq s \leq T$. Note that $\phi_{\varepsilon}(s)$ is symmetric, but is no longer guaranteed to be 
positive on the tails, where oscillations begin to set in. However, what we can guarantee is that the negative overshoots are within $- \varepsilon$. In fact, the quality of the final approximation turns out to be quite satisfactory.
This is illustrated with an example in Figure \ref{fig1}. The main point is that, in the regime $\sigma_r < 15$, we are now able to bring down the order to well within $100$.
We list some of them in Table \ref{table3}. Notice that we can drop more terms for a given accuracy as the kernel gets narrow. For $\sigma_r < 5$, we can drop almost $95\%$ of the terms, while keeping
the error within $0.5\%$ of the peak value.

\begin{algorithm}[!htb]
\caption{Improved \texttt{SHIFTABLE-BF}}
\label{algo2}
\begin{algorithmic}
    \State \textbf{Input}: Image $f(\x)$, variances $\sigma_s^2$ and $\sigma_r^2$, and tolerance $\varepsilon$.
    \State \textbf{Return}: Bilateral filtered image $\tilde{f}(\x)$.
    \State 1. Set $R$ to some factor of $\sigma_s$ (determined by size of spatial Gaussian).
    \State 2. Input $f(\x)$ and $R$ to Algorithm \ref{algo1}, and get $T$.
    \State 3. Set $N =  0. 405 (T/\sigma_r)^2$.
    \State 4. Assign $M$ using the following rule:
        \State \hspace{5mm}   (a)  [$\sigma_r > 40$] Set $M=0$.
        \State \hspace{5mm}   (b) [$10 < \sigma_r \leq 40$] Compute $M$ from \eqref{trunc}, given $N$ and $\varepsilon$.
        \State \hspace{5mm}    (c) [$ \sigma_r \leq 10$]  Plug $N$ and $\varepsilon$ into \eqref{Chernoff} to get $M$.
     \State 5. Use $N$ and $M$ to specify $\phi_{\varepsilon}(\x)$ in \eqref{approx}.
     \State 6. Using $\phi_{\varepsilon}(\x)$ as the range kernel, input $f(\x)$ to \texttt{SHIFTABLE-BF} to get $\tilde{f}(\x)$.
 \end{algorithmic}
\end{algorithm}

Note that \eqref{trunc} actually requires us to compute a large number of tails coefficients, which are eventually not used in \eqref{approx}. It is thus better to estimate $M$ when $N$ is large.
A good estimate of \eqref{trunc} is provided by the Chernoff bound for the binomial distribution \cite{Alon}, namely,
\begin{equation*}
\sum_{n=0}^{M}  2^{-N} \binom{N}{n} \leq \exp\left(-\frac{(N - 2M)^2}{4N} \right).
\end{equation*}
It can be verified the estimate is quite tight for $N > 100$. By setting the bound to $\varepsilon/2$, we get 
\begin{equation}
 \label{Chernoff}
 M = \frac{1}{2}\left(N - \sqrt{4N\log(2/\varepsilon)}\right).
\end{equation}
The final algorithm obtained by combining the proposed modifications is given in Algorithm \ref{algo2}. Henceforth, we will continue to refer to this as the \texttt{SHIFTABLE-BF}. The Matlab implementation of \texttt{SHIFTABLE-BF}
can be found here \cite{Mcode2012}.

\section{Experiments}
\label{experiments}

We now provide some results on synthetic and natural images to understand the improvements obtained used our proposal. While all the experiments were done on Matlab, we took the opportunity to report the run time of a multithreaded Java 
implementation of Algorithm \ref{algo2}. All experiments were run on an Intel quad core $2.83$ GHz processor. 

\subsection{Run time}

First, we tested the speedup obtained using Algorithm \ref{algo1} for computing $T$. We used a Matlab implementation of this algorithm \cite{Mcode2012}. It is expected that the run time remain roughly the same for different $\sigma_s$. As seen in table \ref{table4}, this is indeed the case. The run time of the direct method, for the same image and the same settings of $\sigma_s$, was already provided in table \ref{table2}.
Note that we have been able to cut down the time by a few orders using our fast algorithm. 

\begin{table}[!htbp]
\caption{Average time required to compute the exact value of $T$ using Algorithm \ref{algo1}. We used the standard $512 \times 512$ \textit{Lena}. See also Table \ref{table2}.} 
\label{spread} 
\centering
\begin{tabular}{|c|ccccccc|}
\hline
$\sigma_s$                             &  1   &  3     &  5    &  10    & 15    &  20 & 30  \\
\hline
Time (millisec)         &  70    &  71    &  73   &  71    & 70    &  72 & 71 \\
\hline
\end{tabular}
\label{table4}
\end{table}

\begin{table}[!htbp]
\caption{Comparison of the run times of the multithreaded Java implementations of \texttt{SHIFTABLE-BF} (as in \cite{KcDsMu2011}) and its proposed improvement, for different values of $\sigma_r$ (fixed $\sigma_s= 15$). We used the test image \textit{Checker} shown in Figure \ref{fig2}. 
Shown in the table are the different $\varepsilon$ used in Algorithm \ref{algo2}. We use $\infty$ to signify that the run time is impracticably large.} 
\label{spread} 
\centering
\begin{tabular}{|c|cccccc|}
\hline
$\sigma_r$                                  &  5  &     8 &   10 &   12 &    15 &   20 \\
\hline
\texttt{SHIFTABLE-BF} (millisec)            & $\infty$ &  $\infty$  &  $\infty$ &  2130  &  1280 &  800  \\
\hline
Improved \texttt{SHIFTABLE-BF} (millisec)   &  880     &  500       &  350      &  270   &  250  &  150    \\
\hline
$\varepsilon$  ($\%$ of peak value)         &  3  &  2 &  2 &  1 &   1 &   1     \\
\hline 
\end{tabular}
\label{table5}
\end{table}

We then compared the run times of multithreaded Java implementations of \texttt{SHIFTABLE-BF} proposed in \cite{KcDsMu2011} and its present refinement. For this, we used the test image \textit{Checker} shown in Fig. \ref{fig2}.  
We have used small values of $\sigma_r$, and a fixed $\sigma_s= 15$. The average run times are shown in Table \ref{table5}. The tolerance $\varepsilon$ used for the truncation are also given. 
We use a smaller $\varepsilon$ (larger truncation) as $\sigma_r$ gets small. 
Notice how we have been able to cut down the run time by more than $70\%$. This is not surprising, since we have discarded more than $85\%$ of terms. 
Notice that the run times are now well within $1$ second. The run time of the direct implementation (which does not depend on $\sigma_r$) was around $10$ seconds. The main point is that we can now implement the filter
in a reasonable amount of time for small $\sigma_r$, which could not be done previously in \cite{KcDsMu2011}.

\subsection{Accuracy}

We next studied the effect of truncation. To get an idea of the noise that is injected into the filter due to the truncation, we used the \textit{Checker} image. This particular image allowed us to test both the diffusive and 
the edge-preserving properties of the filter at the same time. We used the setting $\sigma_s= 30$ and $\sigma_r =10$. First, we tried the direct implementation of \eqref{BF}, using a very
fine discretization. Then we tried  Algorithm \ref{algo2}. The difference between the two outputs is shown in Figure \ref{fig3}. The artifacts shown in the image are actually quite insignificant, within $10^{-5}$ times 
the peak value. Notice that most of the artifacts are around the edges. This comes from the oscillations induced at the tails of kernel by the truncation.
To compare the filter outputs (with and without truncation), we extracted two horizontal scan profiles from the respective outputs. These are shown in Figure \ref{fig4}. Notice that it is rather hard to distinguish the two.

\begin{figure}
\centering
\includegraphics[width=0.4\linewidth]{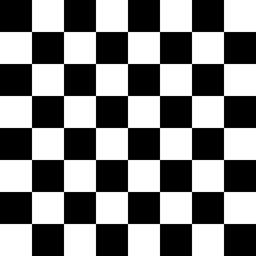} 
\caption{Test image \textit{Checker} of size $256 \times 256$ consisting of black (intensity $0$) and white (intensity $255$) squares. The bilateral filter acts as a diffusion filter in the 
interior of the squares, and as an edge-preserving filter close to the boundary. It preserves both the constant-intensity regions, and the jumps across squares.}
\label{fig2}
\end{figure}

We then applied the filters on the standard grayscale image of \textit{Lena}. We first applied the direct implementation followed by Algorithm \ref{algo2}. In this case, $T$ was computed to be $215$. 
The difference image is shown in Figure \ref{fig5}.  It is again seen that the small artifacts are cluttered near the edges. We have also tried measuring the mean-squared-error (MSE) for different $\sigma_r$. The results are 
given in Table \ref{table6}. Note that relatively larger MSEs are obtained at small $\sigma_r$. This is because we are forced to use a large truncation to speed up the filter at small $\sigma_r$. The above 
results show that we can drastically cut down the run time of filter using the proposed modifications, without incurring significant errors.

\begin{figure}
\centering
\includegraphics[width=0.7\linewidth]{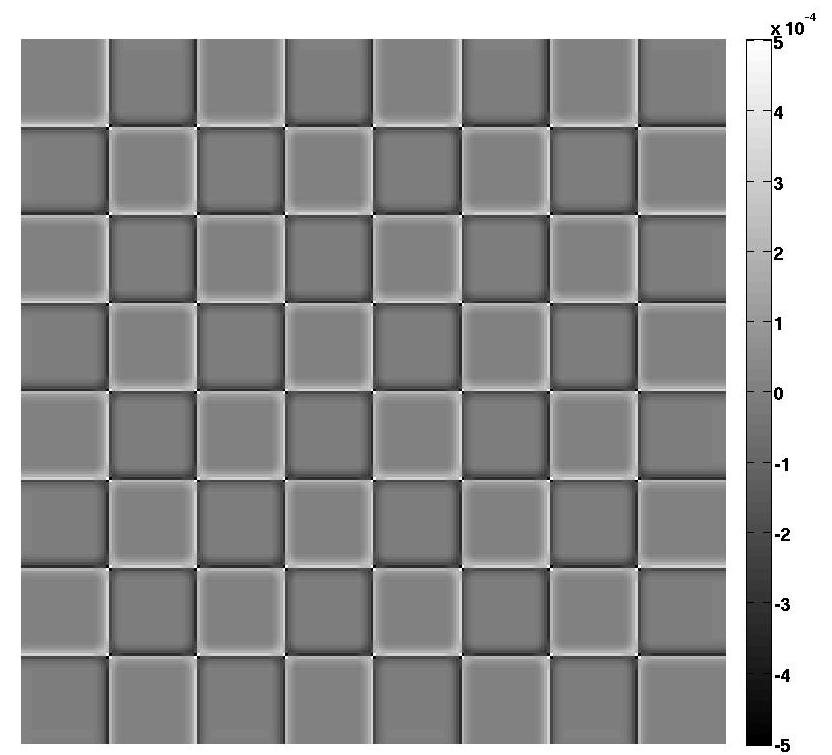} 
\caption{The difference between the outputs of the direct (high resolution) implementation of \eqref{BF}, and that obtained using Algorithm \ref{algo2}. The test image in Figure \ref{fig2}
was used as the input, and settings were $\sigma_s= 30$ and $\sigma_r =10$. The noise created due to the truncation is actually very small -- the error is within  $10^{-5}$ times the peak value. 
See the comparison of scan profiles in Figure \ref{fig4}.}
\label{fig3}
\end{figure}

\begin{figure}[!htp]
  \centering
  \subfloat[High resolution implementation.]{\label{}\includegraphics[width=0.5\linewidth]{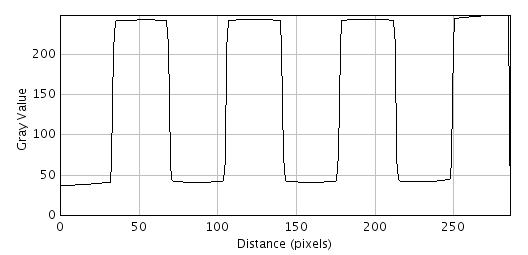}}  \\
  \subfloat[Our implementation.]{\label{}\includegraphics[width=0.5\linewidth]{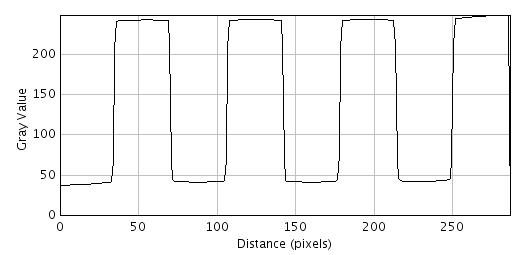}}
  \caption{Comparison of the respective scan profiles from the bilateral filter outputs (cf. description in Figure \ref{fig3}). }
  \label{fig4}
\end{figure}

\begin{figure}
\centering
\includegraphics[width=0.5\linewidth]{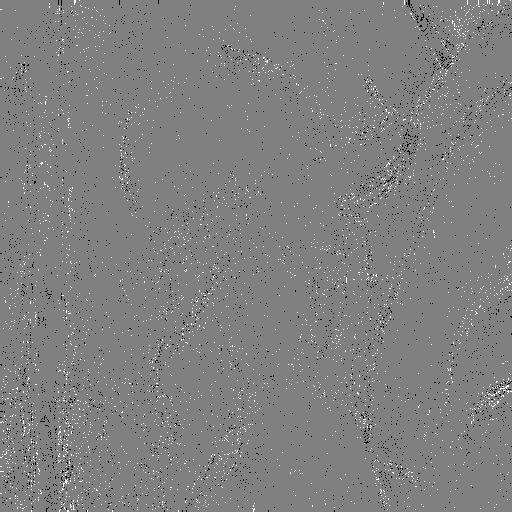} 
\caption{The difference between the outputs for the image \textit{Lena} (size $512 \times 512$). The settings were $\sigma_s= 30$ and $\sigma_r =10$. The noise created due to the truncation 
is within $10^{-4}$ times the peak value, and is thus practically insignificant. Notice that the errors are mainly around the edges.}
\label{fig5}
\end{figure}

\begin{table}[!htbp]
\caption{The mean-squared-error (MSE) between the filter outputs before and after truncation. We use the \textit{Lena} image, and a fixed $\sigma_s=30$. The tolerance $\varepsilon$ is chosen as in Table \ref{table5}.} 
\label{spread} 
\centering
\begin{tabular}{|c|cccccc|}
\hline
$\sigma_r$                    &  5  &     8 &   10 &   12 &    15 &   20 \\
\hline
$10\log_{10}(\mathrm{MSE})$       &  -9.3 &  -11.1 &  -11.8 &  -13.5  &  -13.8 &  -14.1  \\
\hline
\end{tabular}
\label{table6}
\end{table}

\subsection{Comparison with a benchmark algorithm}

We next compared the performance of the improved \texttt{SHIFTABLE-BF} algorithm with those proposed in \cite{Porikli}. The latter algorithms are considered as benchmark in the literature on fast bilateral filtering.
Porikli proposed a couple of algorithms in \cite{Porikli} -- one using a variable spatial filter and a polynomial range filter (we call this \texttt{BF1}), and the other using a constant spatial filter and a variable range 
filter (we call this \texttt{BF2}).
The difficulty with \texttt{BF1} is that it is rather difficult to control the width of the polynomial range filter. 
In particular, as was already mentioned in the introduction, it is difficult to approximate narrow Gaussian range
kernels using \texttt{BF1}. We refer the interested readers to the experimental results in \cite{KcDsMu2011}, where a comparison was already made between \texttt{BF1} and \texttt{SHIFTABLE-BF}. For completeness, we perform a
single experiment to compare these filters when $\sigma_r$ is small (a rather large value of $\sigma_r$ was used in the experiments in \cite{KcDsMu2011}). For this, and the remaining experiments, we will consider the standard test 
image of \textit{Barbara} of size $512 \times 512$. This image has several texture patterns, and is well-suited for comparing the performance of bilateral filters with narrow range kernels\footnote{we thank one of the reviewers
for suggesting this example.}. In particular, we consider the Gaussian
bilateral filter with $\sigma_s = 20$ and $\sigma_r=20$. The results obtained using \texttt{SHIFTABLE-BF} and \texttt{BF1} are shown in Figure \ref{fig6}. The error between the direct implementation of the bilateral filter and
\texttt{SHIFTABLE-BF} was within $10^{-3}$. On the other hand, note how \texttt{BF1} completely breaksdown. The reason for this was already mentioned in the introduction. A similar breakdown, with a larger $\sigma_r$, was 
also observed in Figure 3 in \cite{KcDsMu2011}. 

\begin{figure}[!htp]
  \centering
  \subfloat[Output of \texttt{SHIFTABLE-BF}.]{\label{}\includegraphics[width=0.4\linewidth]{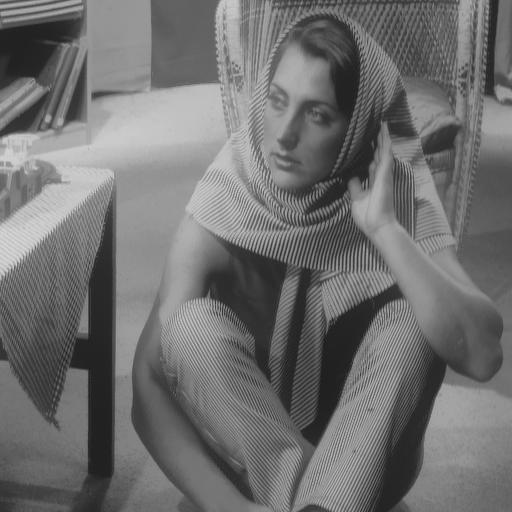}}  
  \subfloat[Output of Porikli's algorithm.]{\label{}\includegraphics[width=0.4\linewidth]{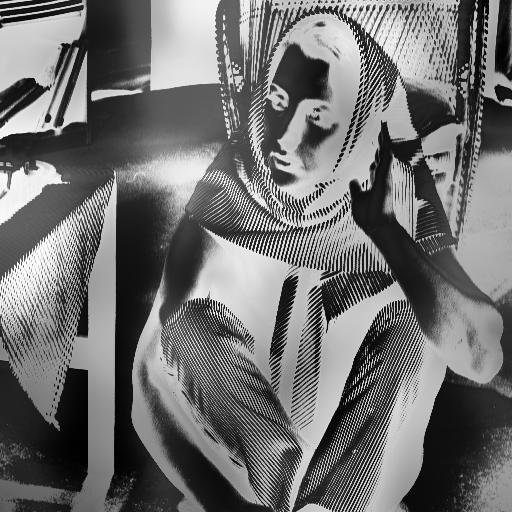}} 
  \caption{Comparison of the bilateral filtering results obtained using \texttt{SHIFTABLE-BF} and Porikli's polynomial-kernel algorithm \cite{Porikli}. In both cases, we used a Gaussian spatial filter ($\sigma_s=20$) and a 
  Gaussian range filter ($\sigma_r = 20$). For our algorithm, we set $\varepsilon = 0.03$. For Porikli's algorithm, we used a Taylor polynomial with order comparable to that of the raised cosine kernel.}
  \label{fig6}
\end{figure}

We next considered \texttt{BF2}, which does not suffer from the above problem. However, we note that \texttt{BF2} cannot be used to perform Gaussian bilater filtering -- it only works with constant spatial filters (box filters).
This is because \texttt{BF2} uses fast integral histograms, and this only works for box filters. To make the comparison even, we considered bilateral filters with constant spatial filters and Gaussian range kernels.
We note \texttt{SHIFTABLE-BF} can be trivially modified to work with arbitrary spatial filters.

\begin{figure}[!htp]
  \centering
  \subfloat[$\sigma_s = 2$.]{\label{}\includegraphics[width=0.5\linewidth]{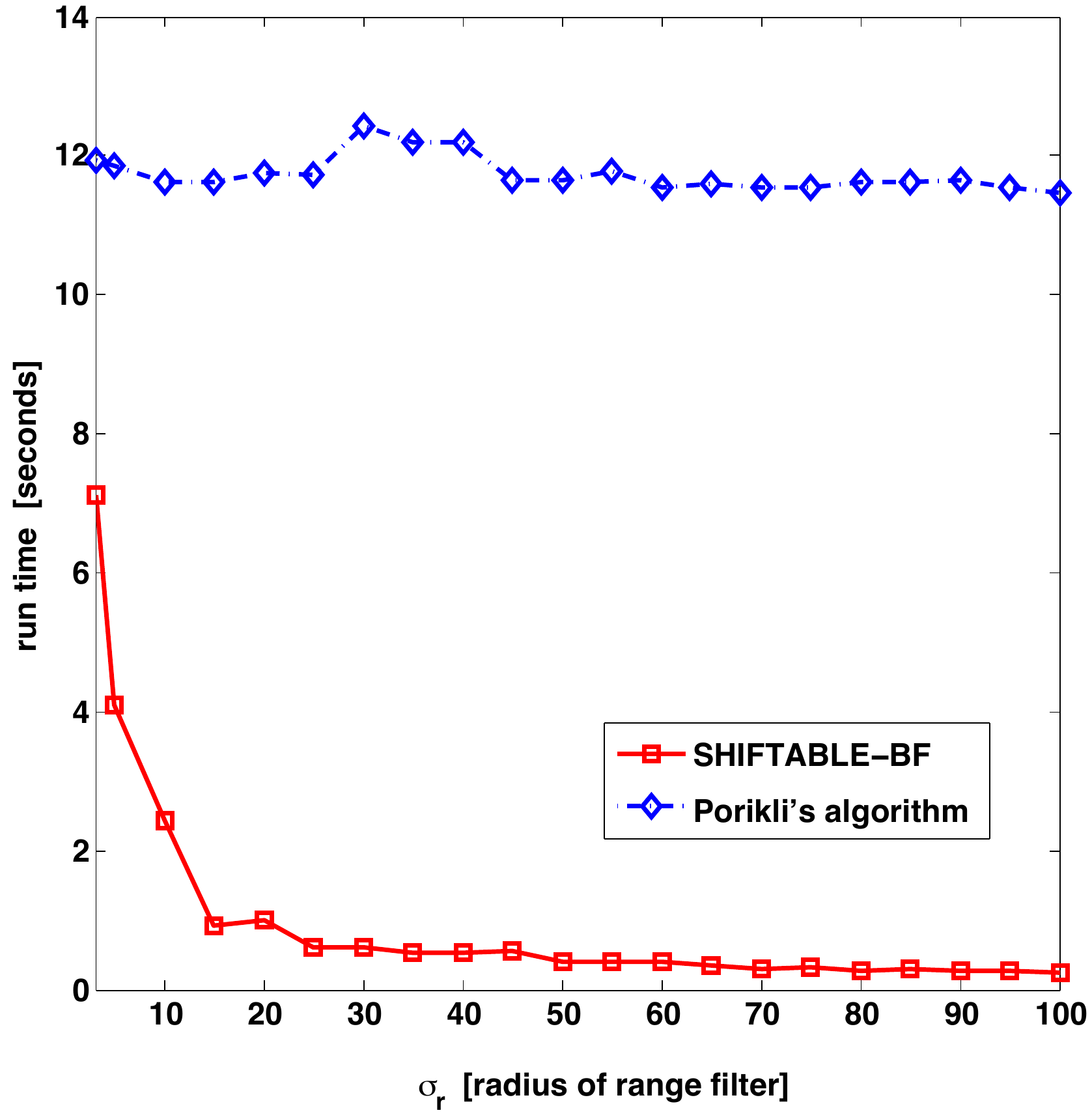}}  
  \subfloat[$\sigma_s = 6$.]{\label{}\includegraphics[width=0.5\linewidth]{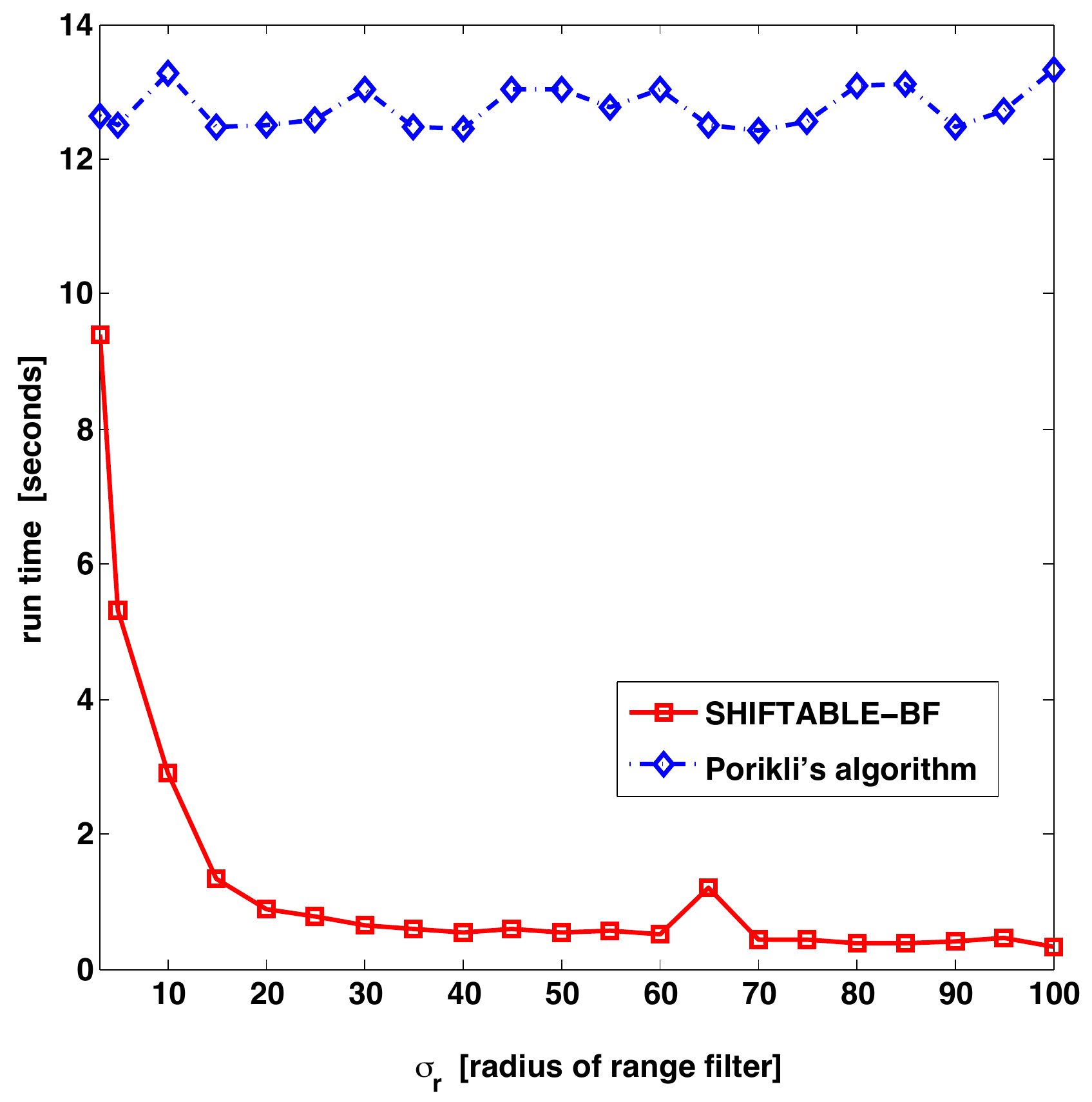}} \\
  \subfloat[$\sigma_s = 10$.]{\label{}\includegraphics[width=0.5\linewidth]{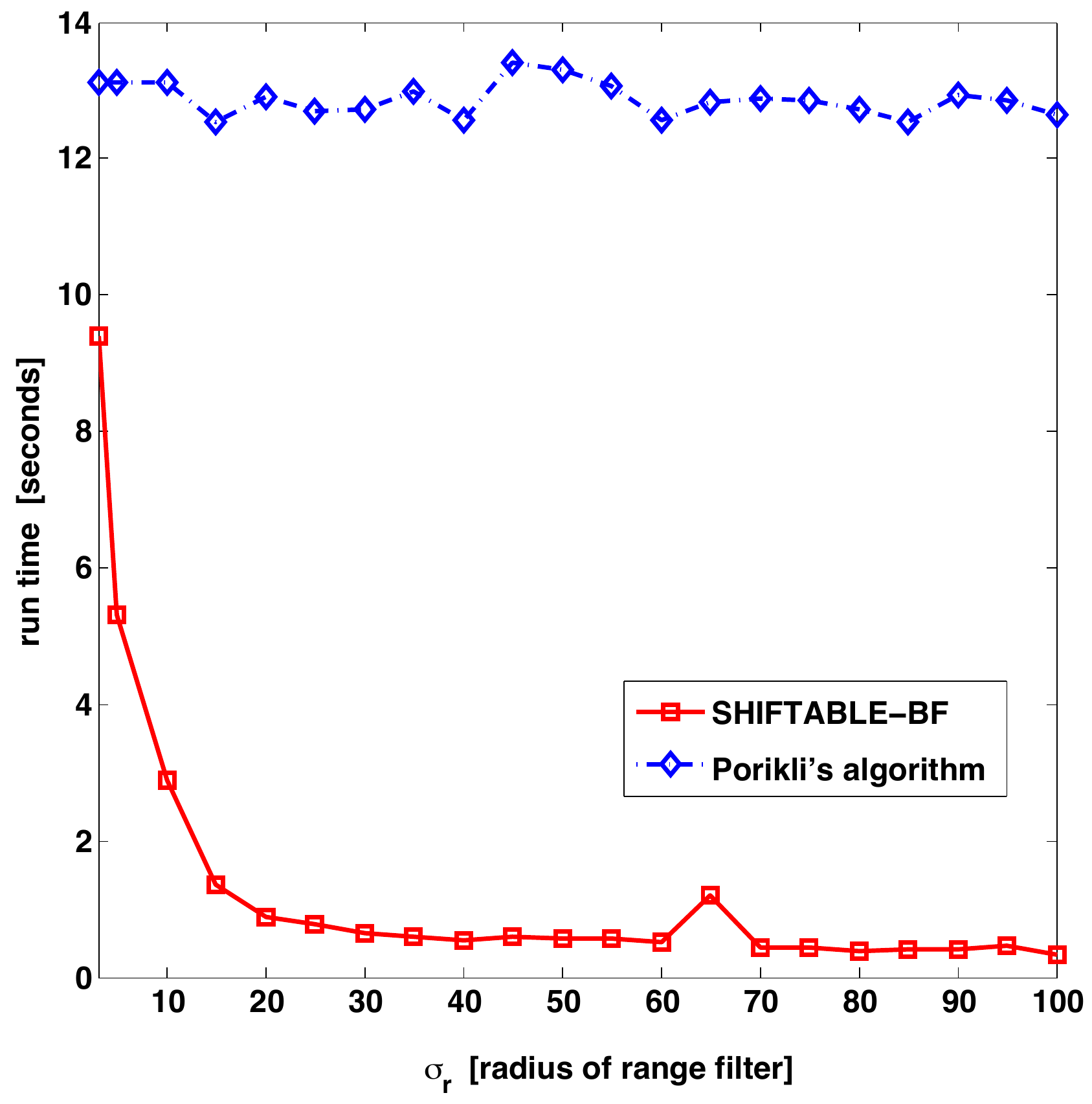}}  
  \subfloat[$\sigma_s = 14$.]{\label{}\includegraphics[width=0.5\linewidth]{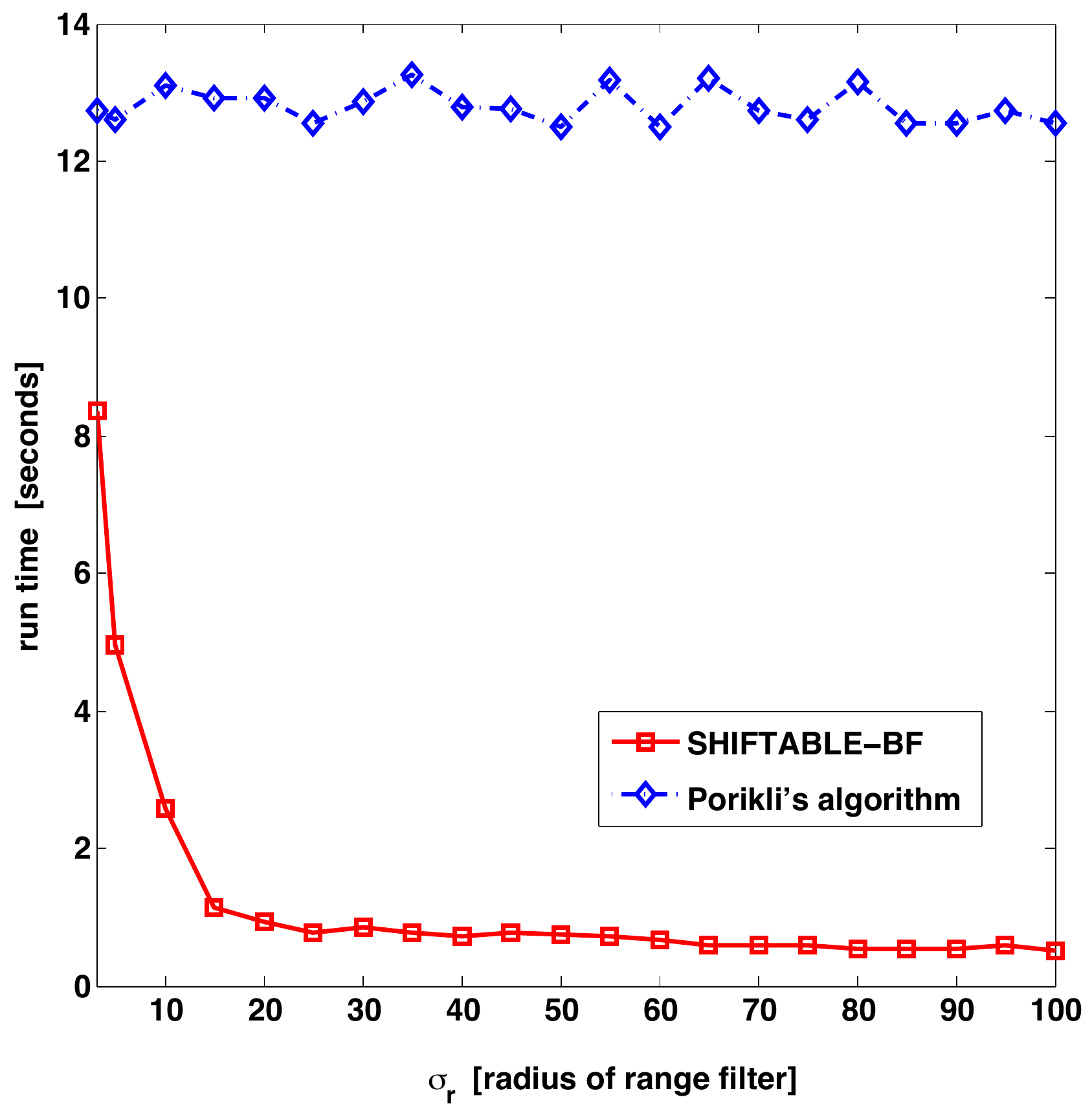}} 
  \caption{Comparison of the run times for different $\sigma_s$ and $\sigma_r$. We compare our algorithm \texttt{SHIFTABLE-BF} with Porikli's algorithm \texttt{BF2} (using integral histograms \cite{Porikli}). 
  In either case, we use a constant spatial filter and a Gaussian range filter. For our algorithm, we set $\varepsilon = 0.01$. For Porikli's algorithm, 
  we use the full resolution ($256$ bin) histogram. Both the algorithms were implemented in Matlab.}
  \label{fig7}
\end{figure}

First, we compared the run times of the Matlab implementations of \texttt{SHIFTABLE-BF} and \texttt{BF2}. The results obtained at particular settings of $\sigma_s$ (radius of box filter) and $\sigma_r$ are shown in 
Figure \ref{fig7}. We see that the run time of \texttt{SHIFTABLE-BF} is consistently better than that of \texttt{BF2}. The difference is particularly large when $\sigma_r > 10$, and it
closes down as $\sigma_r$ gets small. All these can be perfectly explained. Note that, as per the design, the computational complexity of  \texttt{BF2} is $O(1)$ both with respect to $\sigma_s$ and $\sigma_r$. This is 
indeed seen to be the case from the run times. On the other hand, the 
computational complexity of \texttt{SHIFTABLE-BF} is $O(1)$ with respect to $\sigma_s$ (this is again clear from the plots in Figure \ref{fig7}). However, for a given $\sigma_s$, the complexity of the the original algorithm scales as $O(1/\sigma_r^2)$. 
The complexity, in fact, remains roughly the same even after the proposed truncation. 
This explains the step rise in the run time for small values of $\sigma_r$, as shown in Figure \ref{fig7}. However, the actual run time goes down substantially as a result of the truncations (cf. Table \ref{table5}).
In particular, we have noticed that the worst case run time of \texttt{SHIFTABLE-BF} is less than the average run time of \texttt{BF2} for $\sigma_r$ as low as $3$.

\begin{figure}[!htp]
  \centering
  \subfloat[$\sigma_s = 2$.]{\label{}\includegraphics[width=0.5\linewidth]{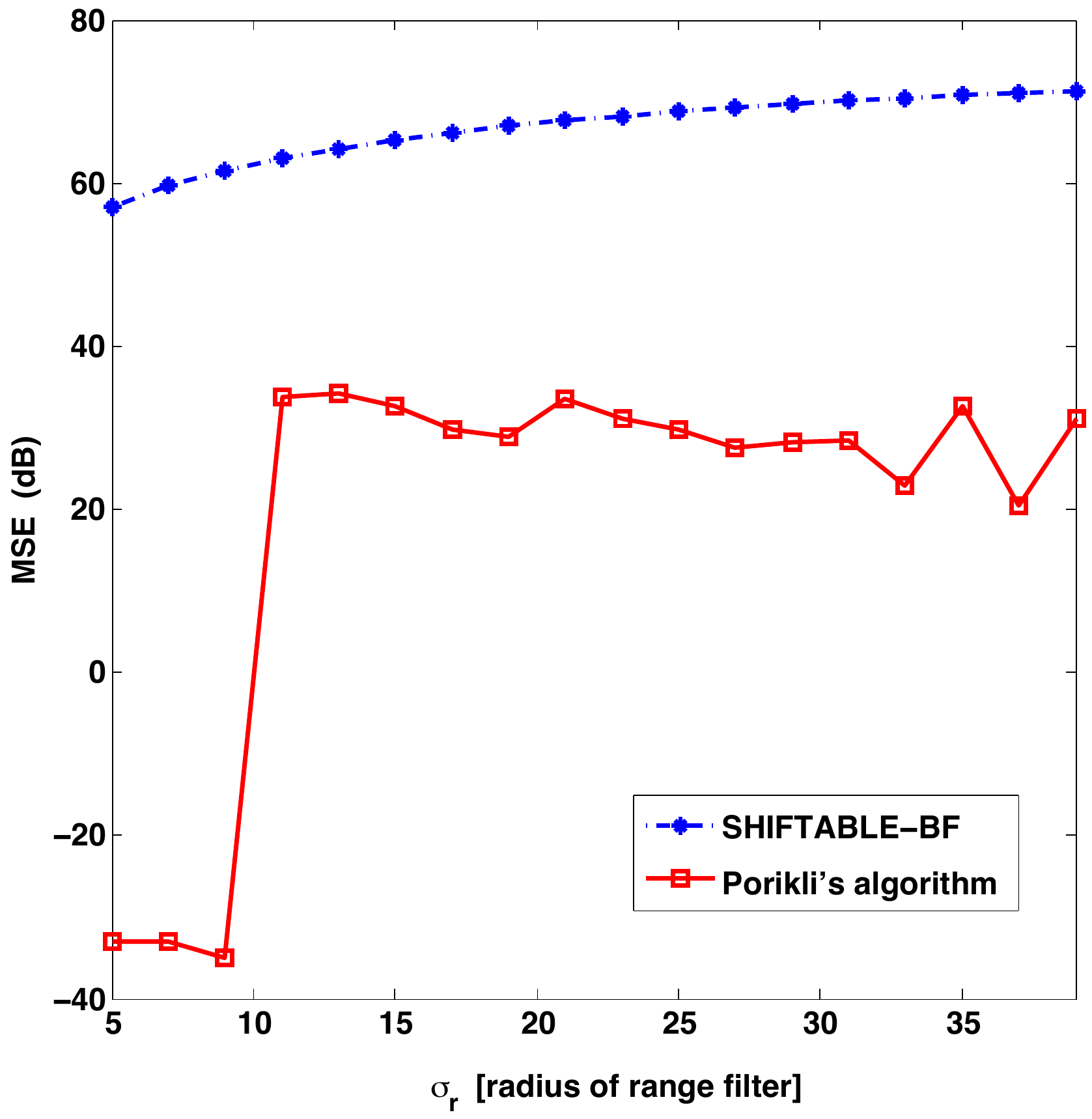}}  
  \subfloat[$\sigma_s = 6$.]{\label{}\includegraphics[width=0.5\linewidth]{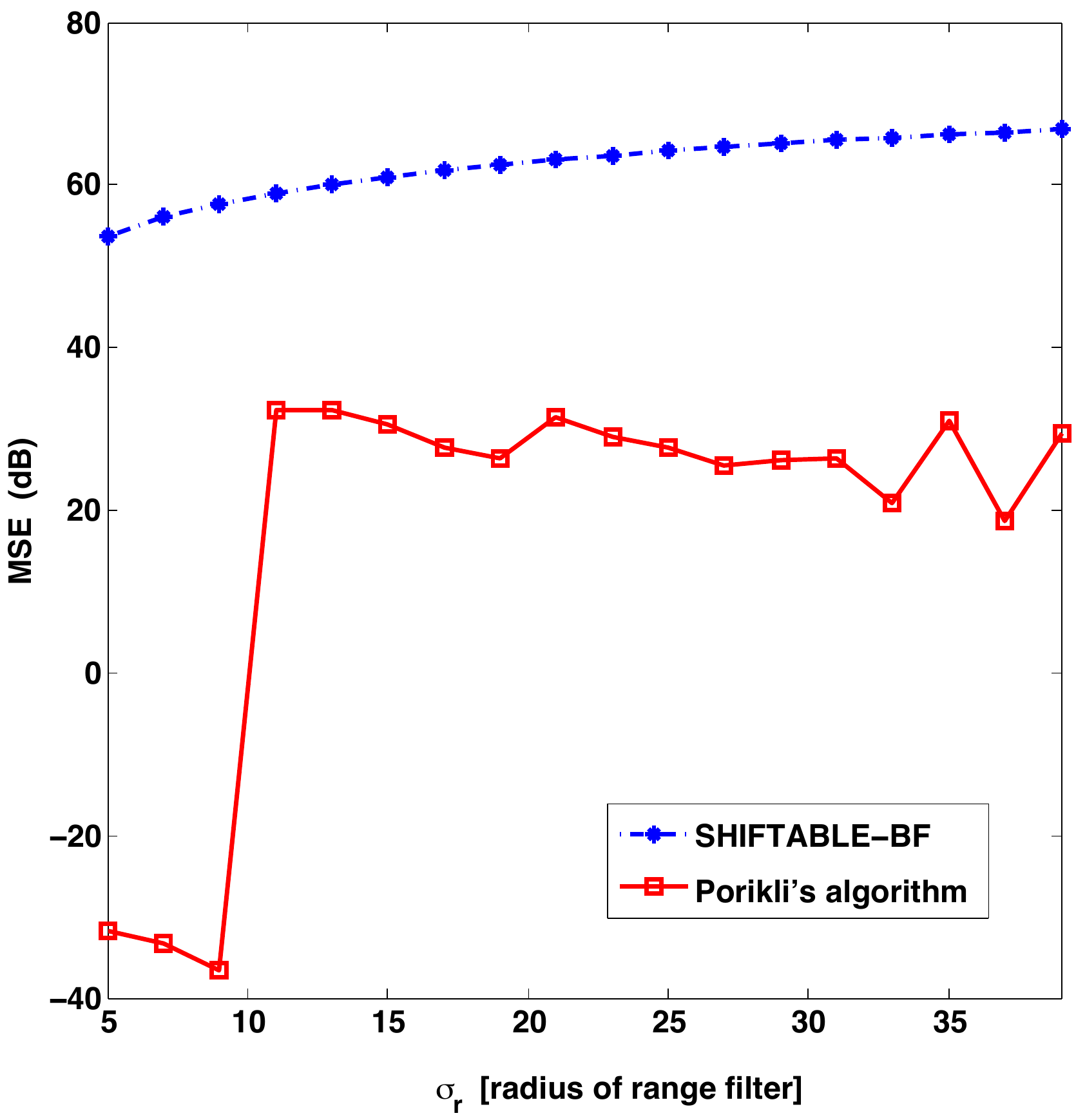}} \\
  \subfloat[$\sigma_s = 10$.]{\label{}\includegraphics[width=0.5\linewidth]{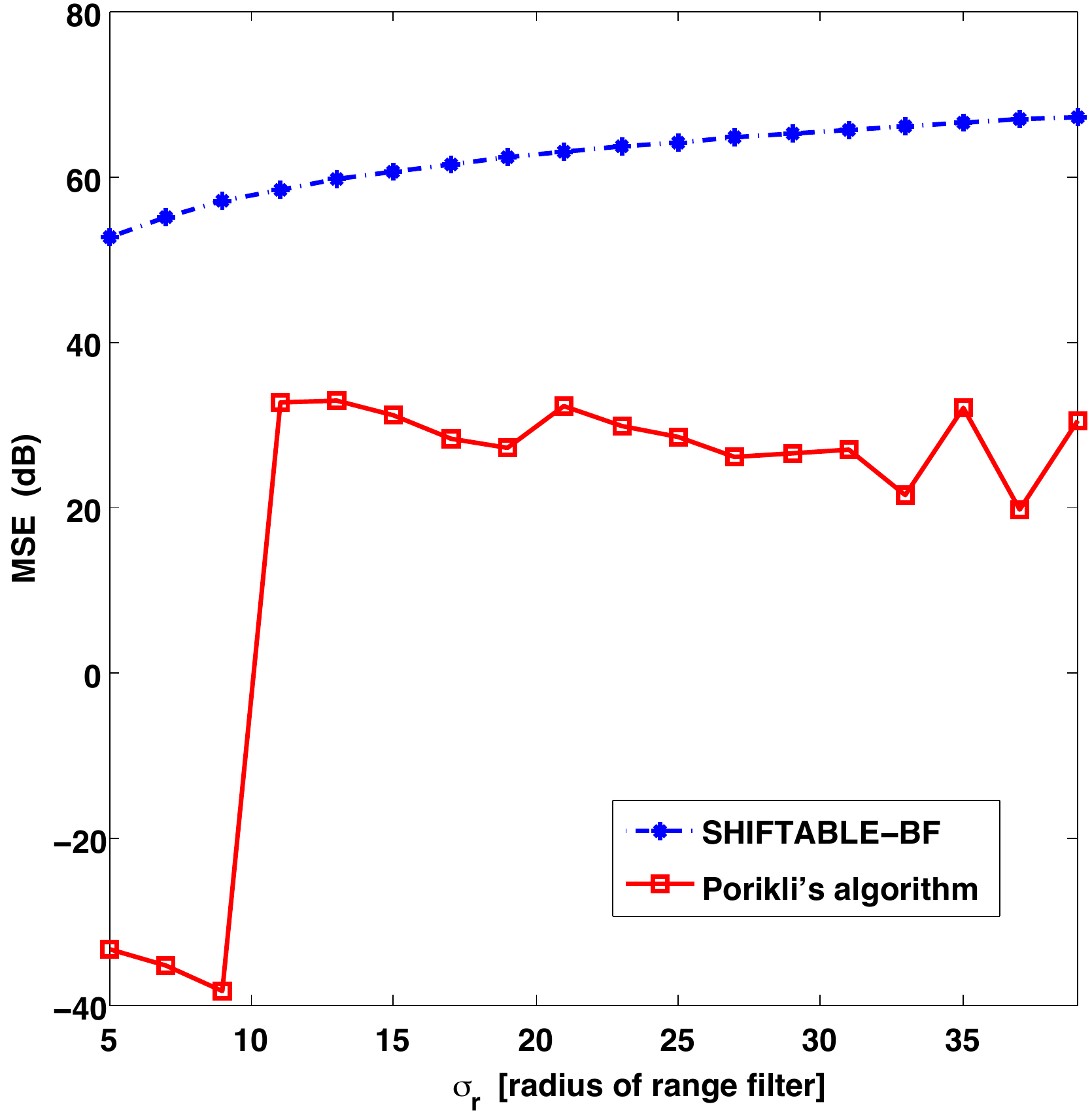}}  
  \subfloat[$\sigma_s = 14$.]{\label{}\includegraphics[width=0.5\linewidth]{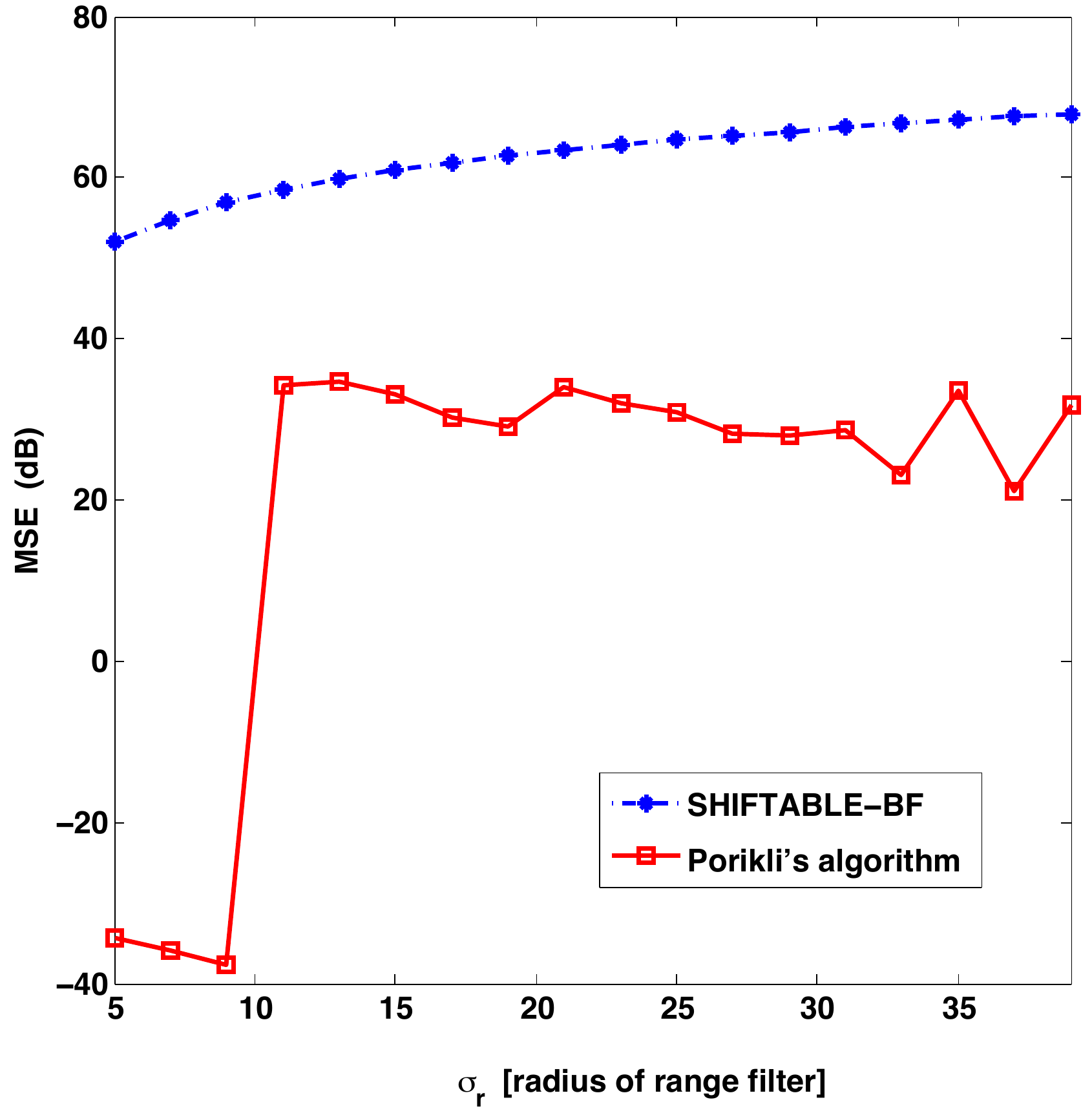}} 
  \caption{Comparison of the MSE for different $\sigma_s$ and $\sigma_r$. The MSEs are computed between the direct implementation and \texttt{SHIFTABLE-BF}, and between the direct implementation and \texttt{BF2}.
  The parameter settings are identical to those used in Figure \ref{fig7}.}
  \label{fig8}
\end{figure}

We note that the run time of \texttt{BF2} depends on the number of bins used for the integral histogram. In the above experiments, we used as many bins as the grayscale levels of the image. It is thus possible to reduce the
run time by cutting down the resolution of the histogram. However, this comes at the cost of the quality of the filtered image. This lead us to compare the outputs of \texttt{SHIFTABLE-BF} and \texttt{BF2}. 
In Figure \ref{fig8}, we compared the MSEs of the two algorithms for different $\sigma_s$ and $\sigma_r$ for the image \textit{Barbara}. We note that the MSE for \texttt{SHIFTABLE-BF} is significantly lower than 
\texttt{BF2}. The gap is around $30$ dB for $\sigma_r > 10$, and around $90$ dB when $\sigma_r \leq 10$. We noticed that this difference becomes even more pronounced if we use a smaller number of bins. As expected, note 
that the individual MSEs do not vary much with $\sigma_s$. The reader will notice that the MSE of \texttt{SHIFTABLE-BF} suddenly 
drops by $60$ dB when $\sigma_r \leq 10$. To explain this, we plot the effective order $N-2M$ for different $\sigma_r$ in Figure \ref{fig9}. We see that the order (of approximation) suddenly jumps up when 
$\sigma_r$ goes below $10$, which explains the jump in the MSE in Figure \ref{fig8}. This is simply due to the rule \eqref{Chernoff} used in Algorithm \ref{algo2} 

\begin{figure}
\centering
\includegraphics[width=0.5\linewidth]{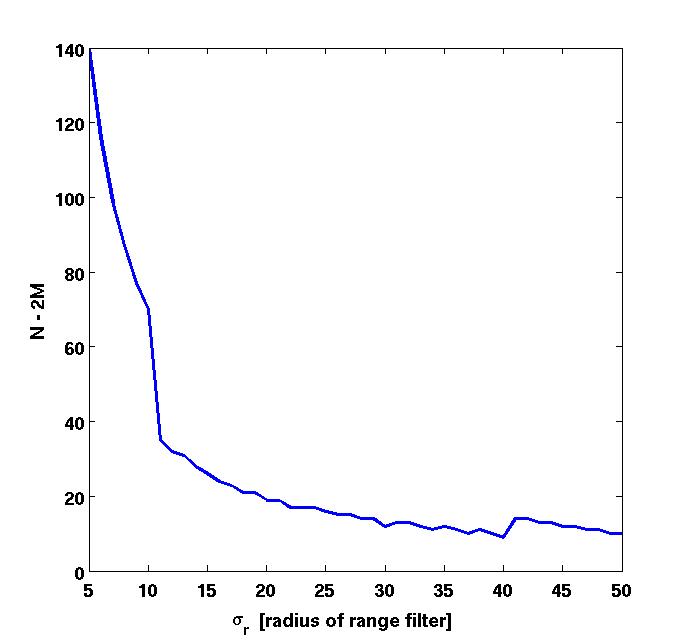} 
\caption{Dependence of order on $\sigma_r$. Here $N$ and $M$ are as defined in Algorithm \ref{algo2}. The sudden jump at $\sigma_r =10$ is due to the truncation rule in \eqref{Chernoff}.}
\label{fig9}
\end{figure}

In Figure \ref{fig10}, the filtered outputs of the two algorithms are compared with the direct implementation, for a small value of $\sigma_r$. Note that the pointwise error between the direct implementation and \texttt{SHIFTABLE-BF} is
of the order $10^{-3}$. On the other hand, the corresponding error between the direct implementation and \texttt{BF2} is substantial, a few orders larger than that for \texttt{SHIFTABLE-BF}. This explains the large gap between
the MSEs in Figure \ref{fig8}.

\begin{figure}[!htp]
  \centering
  \subfloat[Test image \textit{Barbara} (size $512 \times 512$).]{\label{}\includegraphics[width=0.5\linewidth]{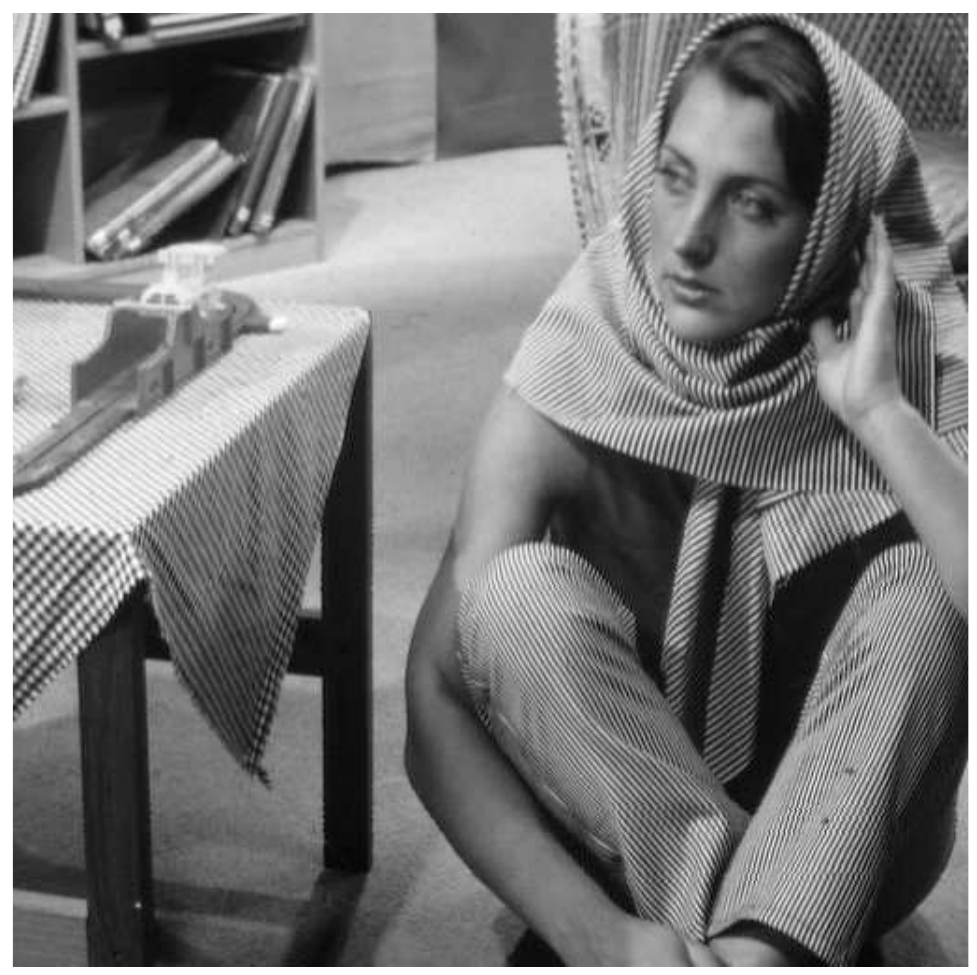}}  
  \subfloat[Bilateral filter output (direct implementation).]{\label{}\includegraphics[width=0.5\linewidth]{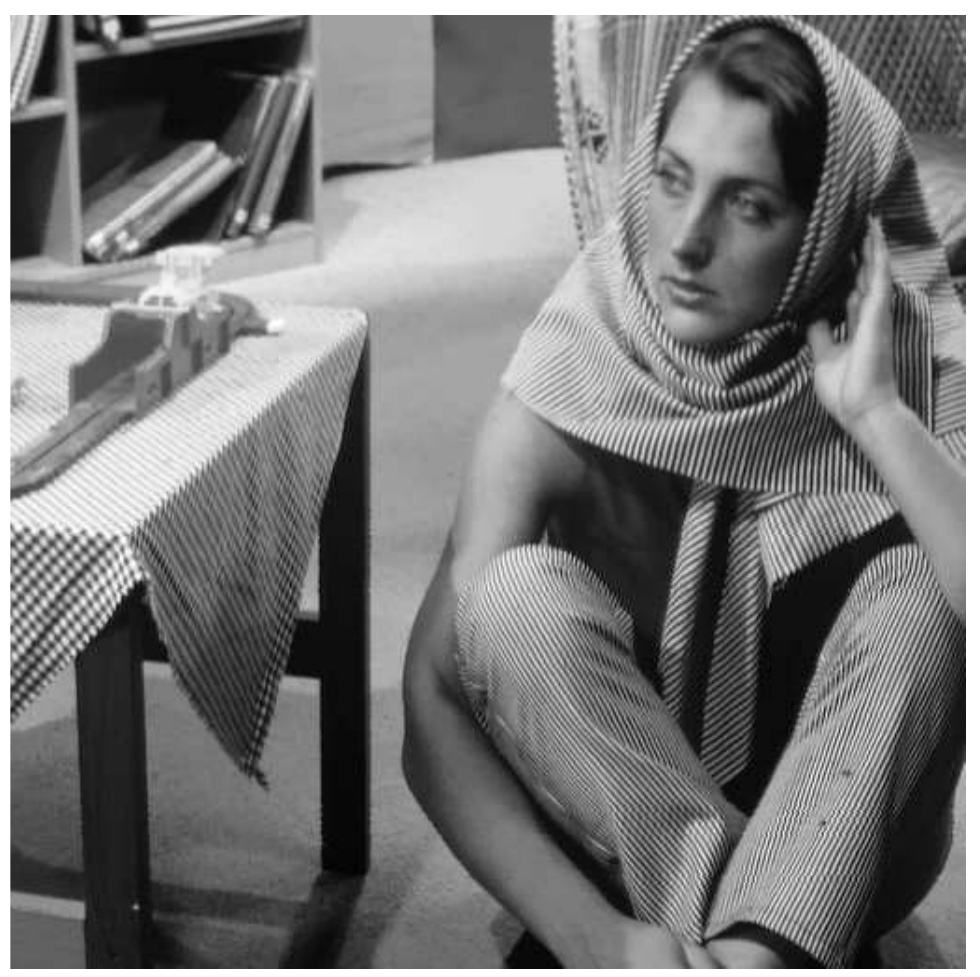}} \\
  \subfloat[Error between (b) and our method.]{\label{}\includegraphics[width=0.5\linewidth]{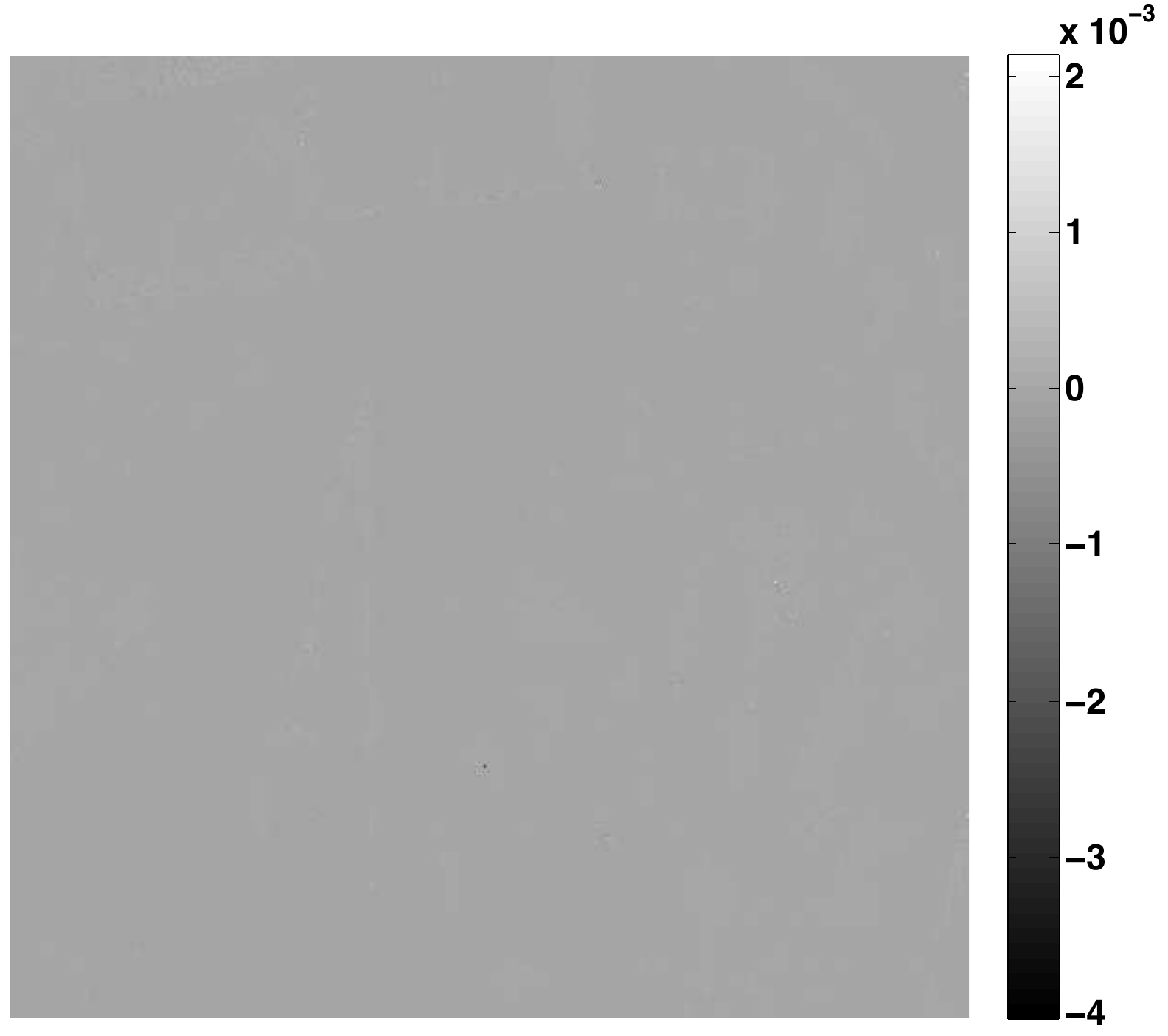}}  
  \subfloat[Error between (b) and Porikli's method.]{\label{}\includegraphics[width=0.5\linewidth]{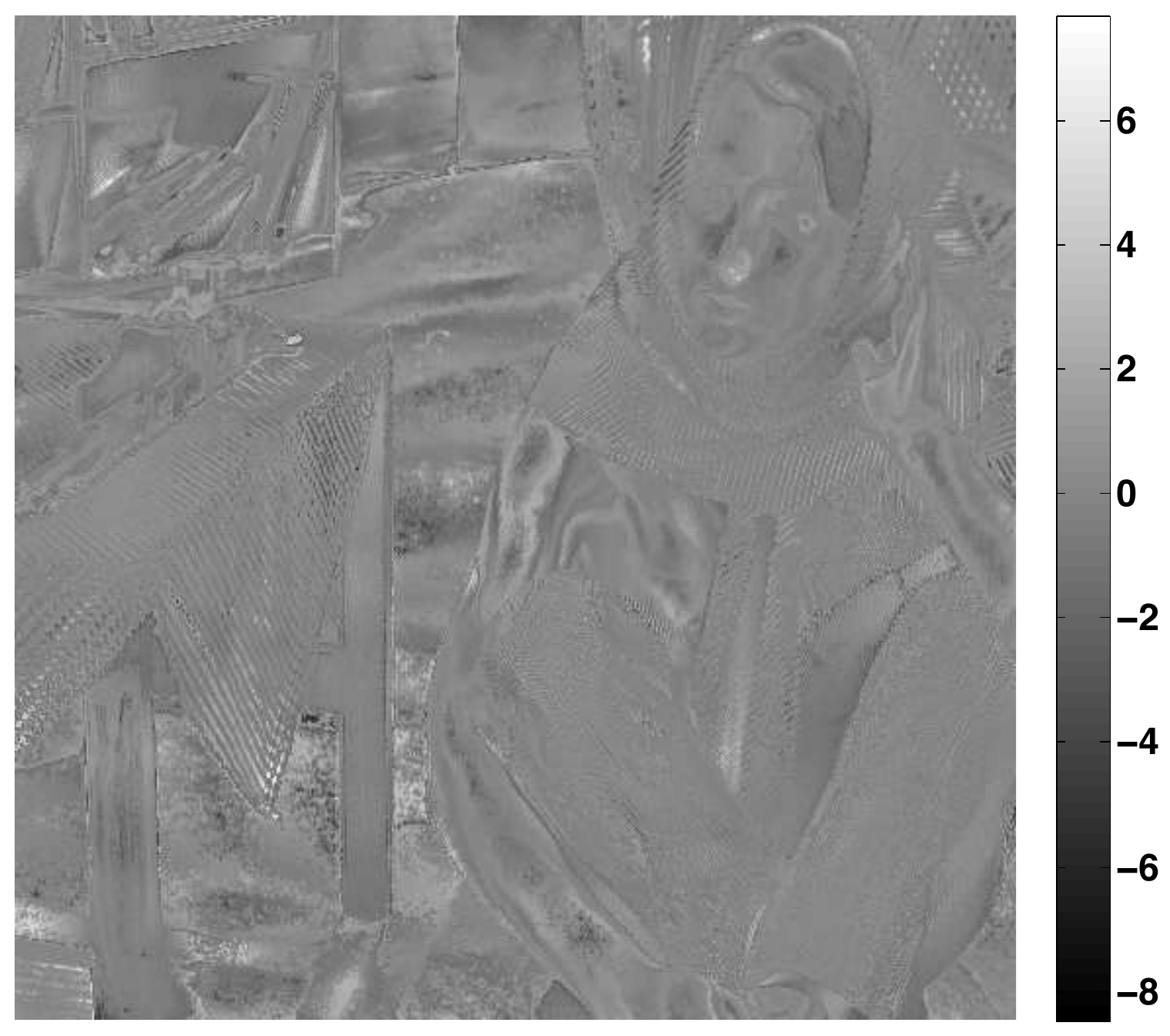}} 
  \caption{Comparison of the results obtained using \texttt{SHIFTABLE-BF} with the direct implementation (high resolution) and 
  Porikli's algorithm \texttt{BF2}. In all three cases, we used a constant spatial filter (radius = $20$) and a Gaussian range filter ($\sigma_r = 5$). For \texttt{SHIFTABLE-BF}, we set $\varepsilon = 0.01$. For \texttt{BF2}, we 
  used the best possible resolution ($256$ bin histogram). The run times of the Matlab implementations were: Direct implementation ($37$ seconds), \texttt{BF2} ($13$ seconds), and \texttt{SHIFTABLE-BF} ($6$ seconds).}
  \label{fig10}
\end{figure}

We close this section by commenting on the memory usage of \texttt{SHIFTABLE-BF} and \texttt{BF2}. The former requires us to compute and store a total of $2(2N-M)$ images, while the latter requires us to 
store a histogram with $B$ bins per pixel (equivalent of $B$ images). It is clear from Figure \ref{fig9} that even for a half-resolution histogram  ($B=128$) and for $\sigma_r > 10$, the memory requirement of \texttt{BF2} is comparable to that of \texttt{SHIFTABLE-BF}. 
For smaller values of $\sigma_r$, \texttt{SHIFTABLE-BF} clearly requires more memory than \texttt{BF2}.

\section{Discussion}

In this paper, we proposed some simple ways of accelerating the bilateral filtering algorithm proposed in \cite{KcDsMu2011}. This, in particular, opened up the possibility of implementing the algorithm in real-time
for small $\sigma_r$. We note that the problem of determining the optimal $\sigma_s$ and $\sigma_r$ for a given application is extrinsic to our algorithm. We are only required to determine the parameter $T$, which is intrinsic
to our algorithm. A fast algorithm was proposed in the paper for this purpose. However, we note that having a fast algorithm does make it easier to 
determine the optimal parameters. In this regard, we note that Kishan et al. have recently shown how our fast algorithm can be used to tune the parameters for image denoising, under different noise models 
\cite{Kishan2012a,Kishan2012b}. One crucial observation used in these papers is that the a certain unbiased estimator of the MSE can be efficiently computed for our fast bilater filter, using the linear
expansions in \eqref{f1} and \eqref{f2}. The ``best'' parameters are choosen by optimizing this MSE estimator. While this can also be done for the polynomial-based bilateral filter in \cite{Porikli}, 
this trick cannot be used for other fast implementations of the bilateral filter, at least to the best of our knowledge.

Finally, we note that the ideas  proposed here can also be extended to the $O(1)$ algorithm for non-local means given in \cite{Kc2011}. In non-local means \cite{Baudes2005}, the range kernel operates on patches 
centered around the pixel of interest. A coarse non-local means was considered in \cite{Kc2011}, where a small patch neighborhood consisting of the pixels $\u_1,\ldots,\u_p$ (where, say, $\u_1=0$) was
used. In this case, the main observation was that formula for the non-local means can be written in terms of following sums:
\begin{align}
\label{NL1}
\int_{\lVert \y \rVert \leq R} f(\x-\y) & g(f(\x+\u_1)-f(\x-\y+\u_1),\ldots, \\ \nonumber
 & f(\x+\u_p)-f(\x-\y+\u_p)) \ d\y,
\end{align}
and
\begin{align}
\label{NL2}
\int_{\lVert \y \rVert \leq R} & g(f(\x+\u_1)-f(\x-\y+\u_1),\ldots, \\ \nonumber
 & f(\x+\u_p)-f(\x-\y+\u_p)) \ d\y,
\end{align}
where $g(s_1,\ldots,s_p)$ is an anisotropic Gaussian in $p$ variables, and has a diagonal covariance. This looks very similar to \eqref{BF}, except that we now
have a multivariate range kernel. By using the separability of $g(s_1,\ldots,s_p)$, and by approximating each Gaussian component by either \eqref{RCconv} or \eqref{RPconv},
a $O(1)$ algorithm for computing \eqref{NL1} and \eqref{NL2} was developed. We refer the readers to \cite{Kc2011} for further details. The key consideration with this algorithm is that the overall order scales as $N^p$, 
where $N$ is the order of the Gaussian approximation for each component. In this case, it is thus important to keep $N$ as low as possible for a given covariance.
After examining \eqref{NL1} and \eqref{NL2}, it is clear that the interval over which the one dimensional Gaussians need to be approximated is $[-T,T]$, where $T$ is as defined in \eqref{bound}.
We can compute this using Algorithm \ref{algo1}. Moreover, we can further reduce the order by truncation, especially when the covariance is small. However, $N^p$ can still be large (even for $p=3$ or $4$),
and hence a parallel implementation must be used for real-time implementation.

\section{Appendix : MAX-FILTER algorithm}

We explain how the \texttt{MAX-FILTER} algorithm works in one dimension. Let $f_1,f_2,\ldots, f_N$ be given, and we have to compute $\max (f_{i-R}, \ldots, f_{i+R})$ at every interior point $i$. Assume $R$ is an integer,
and $N$ is a multiple of the window size $W = 2R+1$, say, $N= pW$ (padding is used if this not the case). The idea is to compute the local maximums using running maximums, similar to running sums used for local 
averaging \cite{Heckbert}. The difference here is that, unlike averaging, the max operation is not linear. This can be fixed using ``local'' running maximums.

We begin by dividing $f_1,f_2,\ldots, f_N$ into $p$ equal partitions. The $k$th partition ($k = 0,1,\ldots,p-1$) is composed of $f_{1 + kW}, \ldots, f_{W + kW}$. 
For a given partition $k$, we recursively compute the two running maximums (of length $W$), one from from the left and one from the right.
Let $l^{(k)}$ and $r^{(k)}$ be the left and right running maximums for the $k$th partition. The sequence $l^{(k)}$ start at the left of the partition with $l^{(k)}_1=f_{1 + kW}$, and 
is recursively given by $l^{(k)}_i =  \max  ( \ l^{(k)}_{i-1}, f_{i + kW} \ )$ for $i=2,3,\ldots,W$. It ends on the right end of the partition. On the other hand, $r^{(k)}$ start at the right and ends on 
the left: $r^{(k)}_W = f_{W + kW}$, and $r^{(k)}_{W-i} =  \max  ( \ r^{(k)}_{W-i+1}, f_{W - i + kW} \ )$ for $i=1,2,\ldots,W-1$. This is done for every partition to get $l^{(0)},\ldots,l^{(p-1)}$ and 
$r^{(0)},\ldots,r^{(p-1)}$. 

We now concatenate the left maximums into a single function $l_1,\ldots,l_N$, that is, we set $l = (l^{(0)},\ldots,l^{(p-1)})$. Similarly, we concatenate the right maximums in order, $r = (r^{(0)},\ldots,r^{(p-1)})$.
In practice, we just need to recursively compute $l_1,\ldots,l_N$ and $r_1,\ldots,r_N$, resetting the recursion at the boundary of every partition. We now split $f_{i-R}, \ldots, f_{i+R}$ into two segments, which 
either belong to the same partition or two adjacent partitions. Then, from the associativity of the max operation, it is seen that $\max (f_{i-R}, \ldots, f_{i+R}) = \max (r_{i-R},l_{i+R})$. Note that, we need just $3$ max operations per point to get the result, independent of the window size $R$.

\section{Acknowledgments}

This work was partly supported by the Swiss National Science Foundation under grant PBELP2-$135867$. The author thanks M. Unser and D. Sage for interesting discussions, and the anonymous referees for their 
helpful comments and suggestions. The author also thanks A. Singer and the Program in Applied and Computational Mathematics at Princeton University for hosting him during this work.

\end{document}